\documentclass[11pt]{article}
\usepackage[margin=1in]{geometry}
    \PassOptionsToPackage{numbers, compress}{natbib}

\usepackage{amsmath,amsfonts,bm}

\def\eqref#1{equation~\ref{#1}}

\def\1{\bm{1}}

\DeclareMathAlphabet{\mathsfit}{\encodingdefault}{\sfdefault}{m}{sl}
\SetMathAlphabet{\mathsfit}{bold}{\encodingdefault}{\sfdefault}{bx}{n}

\usepackage{url}
\usepackage[utf8]{inputenc} %
\usepackage[T1]{fontenc}    %
\usepackage{hyperref}       %
\usepackage{xurl}            %
\usepackage{booktabs}       %
\usepackage{amsfonts}       %
\usepackage{nicefrac}       %
\usepackage{microtype}      %
\usepackage{amsthm}
\usepackage{amsmath}
\usepackage{amssymb}
\usepackage{comment}
\usepackage{graphicx}
\usepackage{caption}
\usepackage{subcaption}
\usepackage{dsfont}
\usepackage{siunitx}
\usepackage{tabularx}
\usepackage{natbib}
\usepackage[dvipsnames]{xcolor} 

\usepackage{threeparttable}
\usepackage{bbm}
\usepackage{adjustbox}
\usepackage{rotating}
\usepackage{titlesec}
\usepackage{multirow,makecell} 
\usepackage{xspace}
\usepackage{bbding}
\usepackage[ruled,vlined,linesnumbered]{algorithm2e} %
\SetKwInput{KwIn}{Input}
\SetKwInput{KwOut}{Output}
\SetAlgoLined
\DontPrintSemicolon
\usepackage{cleveref}
\usepackage{listings}
\usepackage{enumitem}
\usepackage{wrapfig}
\usepackage{tablefootnote}
\definecolor{darksalmon}{RGB}{224,49,49}
\definecolor{tablehead}{RGB}{121,80,242}

\definecolor{green(pigment)}{RGB}{47,158,68}
\definecolor{mydarkblue}{rgb}{0,0.08,0.45}
\definecolor{myblue}{HTML}{3b75c3}
\definecolor{myred}{HTML}{E33222}
\definecolor{mymaroon}{RGB}{142,27,19}
\definecolor{mycite}{cmyk}{0.55,1,0,0.15}
\definecolor{codeblue}{rgb}{0.25,0.5,0.5}
\definecolor{codekw}{rgb}{0.85, 0.18, 0.50}
\definecolor{codegreen}{rgb}{0,0.6,0}
\definecolor{codegray}{rgb}{0.5,0.5,0.5}
\definecolor{codepurple}{rgb}{0.58,0,0.82}
\definecolor{mygray}{gray}{0.925}

\usepackage{xspace}
\newcommand{\ourmethod}{{\fontfamily{lmtt}\selectfont \textbf{PathHD}}\xspace}

\newcommand{\method}{\texttt{CS}\xspace}

\usepackage{color}
\usepackage{colortbl}
\usepackage{pifont}
\definecolor{mydarkblue}{rgb}{0,0.08,0.45}
\definecolor{mycite}{cmyk}{0.55,1,0,0.15}
\definecolor{myblue}{HTML}{00CDCD}
\definecolor{champagne}{rgb}{0.74, 0.83, 0.9}
\definecolor{champagne}{rgb}{0.97, 0.91, 0.81}
\hypersetup{
    colorlinks=true,
    linkcolor=orange,
    citecolor=cyan,
    filecolor=magenta,      
    urlcolor=magenta,
    }

\usepackage[disable]{todonotes}

\newtheorem{innercustomgeneric}{\customgenericname}
\providecommand{\customgenericname}{}
\newcommand{\newcustomtheorem}[2]{%
  \newenvironment{#1}[1]
  {%
   \renewcommand\customgenericname{#2}%
   \renewcommand\theinnercustomgeneric{##1}%
   \innercustomgeneric
  }
  {\endinnercustomgeneric}
}
\newcolumntype{I}{!{\vrule width 1pt}}

\newcustomtheorem{customprop}{Proposition}
\newcustomtheorem{customlemma}{Lemma}
\newcustomtheorem{customremark}{Remark}
\newtheorem{theorem}{Theorem}
\newtheorem{lemma}{Lemma}
\newtheorem{proposition}{Proposition}
\newtheorem{corollary}{Corollary}

\newcommand{\bbind}{\mathbin{\raisebox{0.2ex}{\scalebox{0.9}{$\circledast$}}}} %
\newcommand{\bigbbind}{\mathop{\bbind}}

\theoremstyle{definition}

\usepackage[bottom]{footmisc}

\makeatletter
\def\part{\par
   \addvspace{2ex}
   \@afterindentfalse
   \secdef\@part\@spart}%
\def\@part[#1]#2{%
 \@ifnum{\c@secnumdepth >\m@ne}{%
        \refstepcounter{part}%
        \addcontentsline{toc}{part}{\thepart\hspace{1em}#1}%
 }{%
      \addcontentsline{toc}{part}{#1}%
 }%
   \nobreak
   \vskip 1ex
   \@afterheading
}%
\makeatother

\makeatletter
\newcommand{\thickhline}{%
    \noalign {\ifnum 0=`}\fi \hrule height 1pt
    \futurelet \reserved@a \@xhline
}
\makeatletter
\newenvironment{fullitemize}
{
\vspace{-1pt}
\begin{itemize}[leftmargin=*]
\setlength{\itemsep}{5pt}
\setlength{\parsep}{-5pt}
\setlength{\parskip}{-3pt}
\setlength{\leftmargin}{-10pt}
}
{
\end{itemize}
\vspace{-1pt}
}
\makeatother

\usepackage[loose]{minitoc}
\renewcommand \thepart{}

\doparttoc 
\faketableofcontents

\title{Encoder-Free Knowledge-Graph Reasoning with LLMs via Hyperdimensional Path Retrieval}

\author{
  \textbf{Yezi Liu} \quad
  \textbf{William Youngwoo Chung} \quad
  \textbf{Hanning Chen} \\
  \textbf{Calvin Yeung} \quad
  \textbf{Mohsen Imani} \\
University of California, Irvine \\
  \texttt{\{yezil3,chungwy1,hanningc,chyeung2,m.imani\}@uci.edu}
}

\date{}
\begin{document}

\maketitle
\thispagestyle{empty}  %
\pagestyle{plain}      %
\begin{abstract}
Recent progress in large language models (LLMs) has made knowledge-grounded reasoning increasingly practical, yet KG-based QA systems often pay a steep price in efficiency and transparency.
In typical pipelines, symbolic paths are scored by neural encoders or repeatedly re-ranked by multiple LLM calls, which inflates latency and GPU cost and makes the decision process hard to audit. We introduce \textbf{\ourmethod}, an \emph{encoder-free} framework for knowledge-graph reasoning that couples \emph{hyperdimensional computing} (HDC) with a \emph{single} LLM call per query.
Given a query, \ourmethod represents relation paths as block-diagonal \emph{GHRR} hypervectors, retrieves candidate paths using a calibrated \emph{blockwise cosine similarity} with Top-$K$ pruning, and then performs a one-shot LLM adjudication that outputs the final answer together with supporting, citeable paths. The design is enabled by three technical components: (i) an order-sensitive, non-commutative binding operator for composing multi-hop paths, (ii) a robust similarity calibration that stabilizes hypervector retrieval, and (iii) an adjudication stage that preserves interpretability while avoiding per-path LLM scoring.
Across WebQSP, CWQ, and GrailQA, \ourmethod matches or improves Hits@1 compared to strong neural baselines while using only one LLM call per query, reduces end-to-end latency by \textbf{40-60\%}, and lowers GPU memory by \textbf{3-5$\times$} due to encoder-free retrieval.
Overall, the results suggest that carefully engineered HDC path representations can serve as an effective substrate for efficient and faithful KG-LLM reasoning, achieving a strong accuracy-efficiency-interpretability trade-off.
\end{abstract}

\section{Introduction}\label{sec:intro}
Large Language Models (LLMs) have rapidly advanced reasoning over both text and structured knowledge. Typical pipelines follow a \emph{retrieve–then–reason} pattern: they first surface evidence (documents, triples, or relation paths), then synthesize an answer using a generator or a verifier~\citep{lewis2020rag,press2022selfask,yao2023react,wei2022chain,yao2024tree}. In knowledge-graph question answering (KGQA), this often becomes \emph{path-based reasoning}: systems construct candidate relation paths that connect the topic entities to potential answers and pick the most plausible ones for final prediction~\citep{sun2018open,jiang2022unikgqa,jiang2023structgpt,jiang2024kg,luo2023reasoning}. While these approaches obtain strong accuracy on WebQSP, CWQ, and GrailQA, they typically depend on heavy neural encoders (e.g., Transformers or GNNs) or repeated LLM calls to rank paths, which makes them slow and expensive at inference time, especially when many candidates must be examined.

\begin{figure}[t] \centering \includegraphics[width=\linewidth]{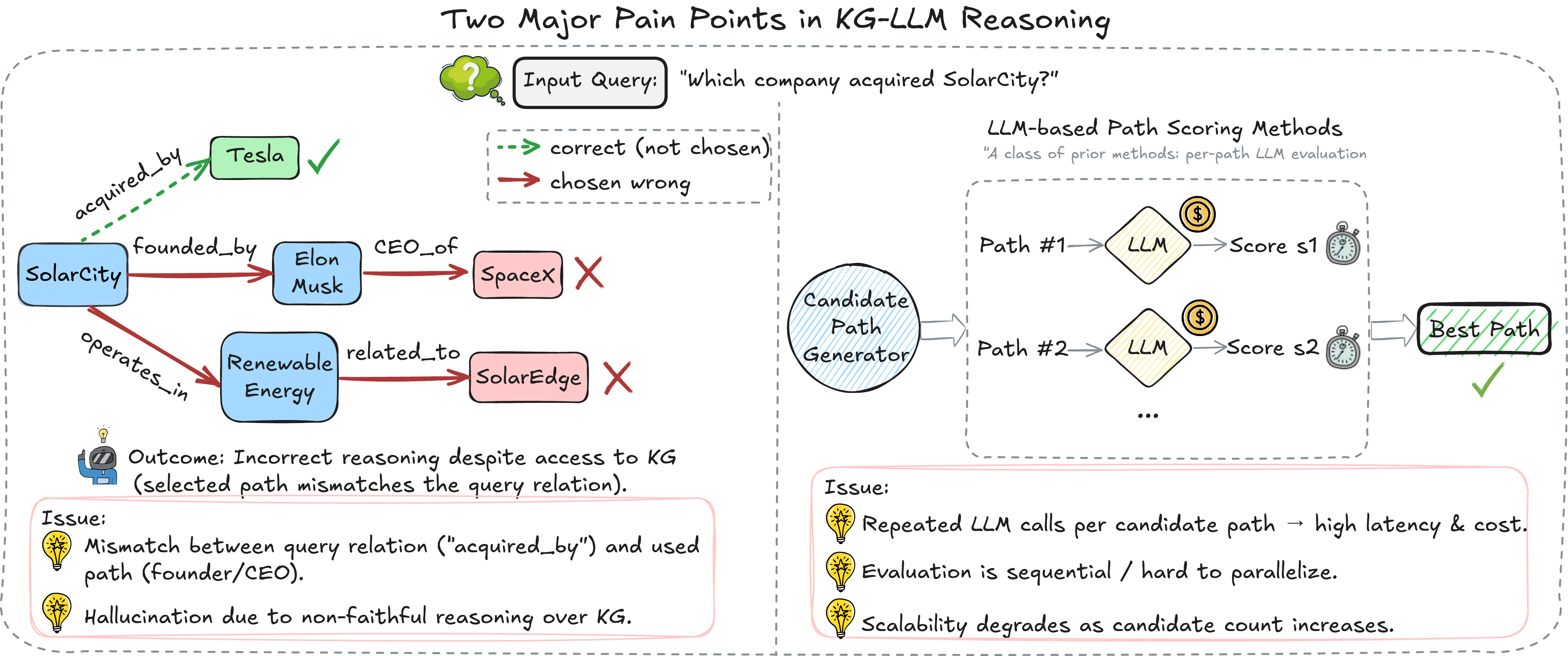} 
\vspace{-1.5em}
\caption{
\textbf{Two major pain points in KG-LLM reasoning.}
\textbf{Left:} A path-based KGQA system selects a candidate path whose relation sequence does not match the query relation (``acquired\_by''), leading to an incorrect answer even though the KG contains the correct evidence.
\textbf{Right:} LLM-based scoring evaluates each candidate path in a separate LLM call, which is sequential, hard to parallelize, and incurs high latency and token cost as the candidate set grows.
} \label{fig:intro-pain} 
\vspace{-1em}
\end{figure}

\Cref{fig:intro-pain} highlights two recurring issues in KG–LLM reasoning.
\ding{182} \textbf{Path–query mismatch:} Order-insensitive encodings, weak directionality, and noisy similarity often favor superficially related yet misaligned paths, blurring the question’s intended relation.
\ding{183} \textbf{Per-candidate LLM scoring:} Many systems score candidates sequentially, so latency and token cost grow roughly linearly with set size; batching is limited by context/API, and repeated calls introduce instability, yet models can still over-weight long irrelevant chains, hallucinate edges, or flip relation direction. Most practical pipelines first detect a topic entity, enumerate $10\!\sim\!100$ length-$1$–$4$ paths, then score each with a neural model or LLM, sending top paths to a final step~\citep{sun2018open,luo2023reasoning,jiang2024kg}. This hard-codes two inefficiencies: (i) \emph{neural scoring dominates latency} (fresh encoding/prompt per candidate), and (ii) \emph{loose path semantics} (commutative/direction-insensitive encoders conflate \textit{founded\_by}$\!\to$\textit{CEO\_of} with its reverse), which compounds on compositional/long-hop questions.

Hyperdimensional Computing (HDC) offers a different lens: represent symbols as long, nearly-orthogonal \emph{hypervectors} and manipulate structure with algebraic operations such as \emph{binding} and \emph{bundling}~\citep{kanerva2009hyperdimensional,plate1995holographic}. HDC has been used for fast associative memory, robust retrieval, and lightweight reasoning because its core operations are elementwise or blockwise and parallelize extremely well on modern hardware~\citep{frady2021variable}. Encodings tend to be noise-tolerant and compositional; similarity is computed by simple cosine or dot product; and both storage and computation scale linearly with dimensionality. Crucially for KGQA, HDC supports \emph{order-sensitive} composition when the binding operator is non-commutative, allowing a path like $r_1\!\to r_2\!\to r_3$ to be distinguished from its permutations while remaining a single fixed-length vector. This makes HDC a promising substrate for ranking many candidate paths without invoking a neural model for each one.

Motivated by these advantages, we introduce \textbf{\ourmethod} (Hyper\textbf{D}imensional \textbf{Path} Retrieval), a lightweight retrieval-and-reason framework for efficient KGQA with LLMs. First, we map every relation to a block-diagonal unitary representation and encode a candidate path by \emph{non-commutative} Generalized Holographic Reduced Representation (GHRR) binding~\citep{yeung2024generalized}; this preserves order and direction in a single hypervector. In parallel, we encode the query into the same space to obtain a \emph{query hypervector}. Second, we score \emph{all} candidates via cosine similarity to the query hypervector and keep only the top-$K$ paths with a simple, parallel Top-$K$ selection. Finally, instead of per-candidate LLM calls, we make \emph{one} LLM call that sees the question plus these top-$K$ paths (verbalized), and it outputs the answer along with cited supporting paths. In effect, \ourmethod addresses both pain points in~\Cref{fig:intro-pain}: order-aware binding reduces path–query mismatch, and vector-space scoring eliminates per-path LLM evaluation, cutting latency and token cost.

Our contributions are as follows:
\begin{fullitemize}
  \item[\ding{182}] \textbf{A fast, order-aware retriever for KG paths.} We present \ourmethod, which uses GHRR-based, non-commutative binding to encode relation \emph{sequences} into hypervectors and ranks candidates with plain cosine similarity, without learned neural encoders or per-path prompts. This design keeps a symbolic structure while enabling fully parallel scoring with $\mathcal{O}(Nd)$ complexity.
  \item[\ding{183}] \textbf{An efficient one-shot reasoning stage.} \ourmethod replaces many LLM scoring calls with a single LLM adjudication over the top-$K$ paths. This decouples retrieval from generation, lowers token usage, and improves wall-clock latency while remaining interpretable: the model cites the supporting path(s) it used.
  \item[\ding{184}] \textbf{Extensive validation and operator study.} On WebQSP, CWQ, and GrailQA, \ourmethod achieves competitive Hits@1 and F1 with markedly lower inference cost. An ablation on binding operators shows that our block-diagonal (GHRR) binding outperforms commutative binding and circular convolution, and additional studies analyze the impact of top-$K$ pruning and latency–accuracy trade-offs.
\end{fullitemize}

Overall, {\ourmethod} shows that carefully designed hyperdimensional representations can act as an encoder-free, training-free path scorer inside KG-based LLM reasoning systems, preserving competitive answer accuracy while substantially improving inference efficiency and providing explicit path-level rationales.

\section{Method}\label{sec:pipeline}
The proposed \ourmethod follows a \emph{Plan $\rightarrow$ Encode $\rightarrow$ Retrieve $\rightarrow$ Reason} pipeline (\Cref{fig:framework}). 
(i) We first generate or select relation \emph{plans} that describe how an answer can be reached (schema enumeration optionally refined by a light prompt). 
(ii) Each plan is mapped to a hypervector via a non-commutative GHRR binding so that order and direction are preserved. 
(iii) We compute a blockwise cosine similarity in the hypervector space and apply Top-$K$ pruning. 
(iv) Finally, a \emph{single} LLM call produces the answer with path-based explanations. 
This design keeps the heavy lifting in cheap vector operations, delegating semantic adjudication to one-shot LLM reasoning.

\subsection{Problem Setup \& Notation}
Given a question $q$, a knowledge graph (KG) $\mathcal{G}$, and a set of relation schemas $\mathcal{Z}$, the goal is to predict an answer $a$. Formally, we write $\mathcal{G} = (\mathcal{V}, \mathcal{E}, \mathcal{R})$, where $\mathcal{V}$ is the set of entities, $\mathcal{R}$ is the set of relation types, and $\mathcal{E} \subseteq \mathcal{V} \times \mathcal{R} \times \mathcal{V}$ is the set of directed edges $(e, r, e')$. We denote entities by $e \in \mathcal{V}$ and relations by $r \in \mathcal{R}$. A relation schema $z \in \mathcal{Z}$ is a sequence of relation types $z = (r_1,\dots,r_\ell)$. Instantiating a schema $z$ on $\mathcal{G}$ yields concrete KG paths of the form $(e_0, r_1, e_1, \dots, r_\ell, e_\ell)$ such that $(e_{i-1}, r_i, e_i) \in \mathcal{E}$ for all $i$. For a given question $q$, we denote by $\mathcal{P}(q)$ the set of candidate paths instantiated from schemas in $\mathcal{Z}$ and by $N = |\mathcal{P}(q)|$ its size. We write $d$ for the dimensionality of the hypervectors used to represent relations and paths.

A key challenge is to efficiently locate a small set of \emph{plausible} paths for $q$ from this large candidate pool, and then let an LLM reason over only those paths. A summary of the notation throughout the paper can be found in~\Cref{app:not}.

\begin{figure}[t]
  \centering
  \includegraphics[trim=30 300 30 20,clip,width=\linewidth]{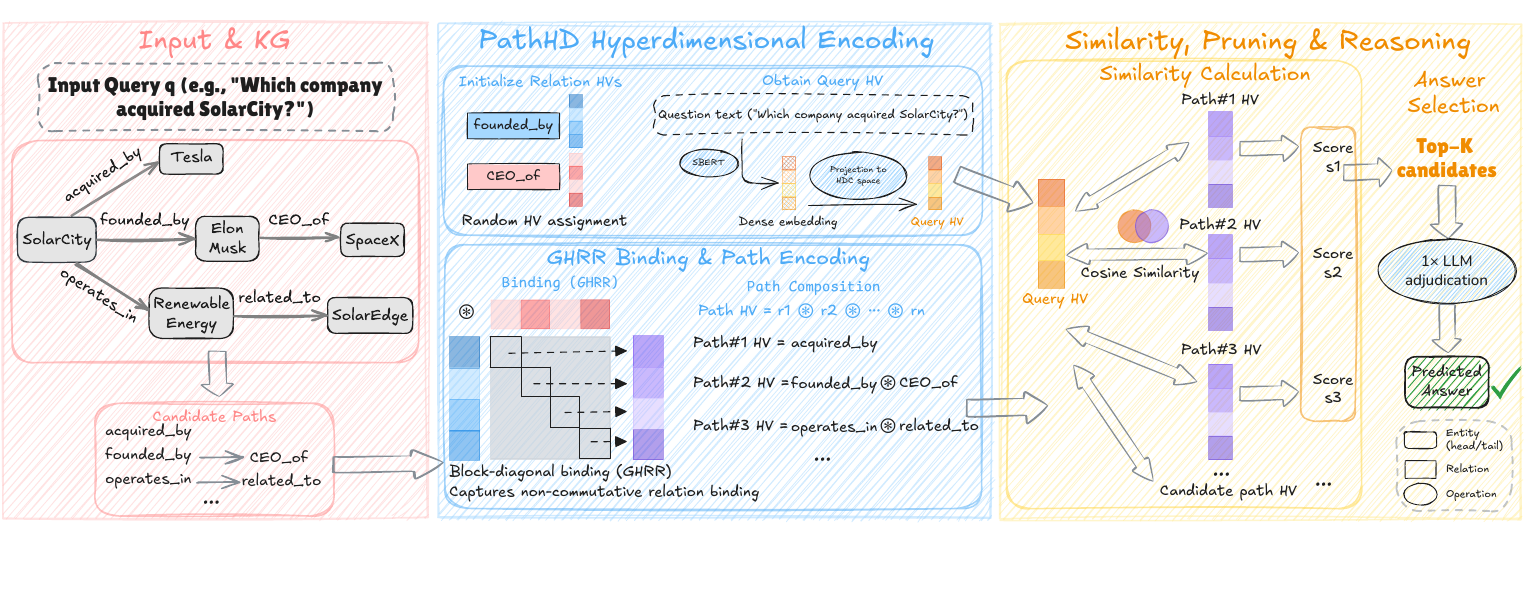}
  \vspace{-3em}
\caption{
Overview of {\ourmethod}: a \emph{Plan} $\rightarrow$ \emph{Encode} $\rightarrow$
\emph{Retrieve} $\rightarrow$ \emph{Reason} pipeline.
A schema-based planner first generates relation plans over the KG; {\ourmethod} encodes these
plans and instantiates candidate paths into order-aware GHRR hypervectors, ranks
candidates with blockwise cosine similarity and Top-$K$ selection, and then issues a
\emph{single} LLM adjudication call to answer with cited paths, with most computation
handled by vector operations and modest LLM use.
}\label{fig:framework}
\end{figure}

\subsection{Hypervector Initialization}
\label{sec:hv-init}
We work in a Generalized Holographic Reduced Representations (GHRR) space.
Each atomic symbol $x$ (relation or, optionally, entity) is assigned a
$d$-dimensional hypervector $\mathbf{v}_x\!\in\!\mathbb{C}^{d}$ constructed as a
block vector of unitary matrices:
\begin{equation}
\mathbf{v}_x \;=\; [A^{(x)}_{1};\dots;A^{(x)}_{D}],\qquad
A^{(x)}_{j}\in \mathrm{U}(m),\;\; d=Dm^2.
\end{equation}
In practice, we sample each block from a simple unitary family for efficiency, e.g.,
\[
A^{(x)}_{j}=\operatorname{diag}\!\big(e^{i\phi_{j,1}},\ldots,e^{i\phi_{j,m}}\big),\qquad
\phi_{j,\ell}\sim\operatorname{Unif}[0,2\pi),
\]
or a random Householder product.
Blocks are $\ell_2$-normalized so that all hypervectors have unit norm.
This initialization yields near-orthogonality among symbols, which concentrates
with dimension (cf. Prop.~\ref{prop:near-orth}). Hypervectors are sampled once
and kept fixed; the retriever itself has no learned parameters.

\paragraph{Query hypervector.}
For a question $q$, we obtain a query hypervector in two ways, depending on the
planning route used in \Cref{sec:pipeline}:
(i) \emph{plan-based}—encode the selected relation plan
$z_q=(r_1,\dots,r_\ell)$ using the same GHRR binding as paths
(see Eq.~\eqref{eq:ghrr-path}); or
(ii) \emph{text-projection}—embed $q$ with a sentence encoder (e.g., SBERT) to
$\mathbf{h}_q \in \mathbb{R}^{d_t}$ and project it to the HDC space using a
fixed random linear map $P \in \mathbb{R}^{d \times d_t}$, then block-normalize:
\begin{equation}
\mathbf{v}_q \;=\; \mathcal{N}_{\text{block}}\!\big(P\,\mathbf{h}_q\big).
\label{eq:query-hv}
\end{equation}
Both choices produce a query hypervector compatible with GHRR scoring.
Unless otherwise specified, we use the plan-based encoding in all main
experiments, and report the text-projection variant only in ablations
(\Cref{app:ablation_text_pro}).

\subsection{GHRR Binding and Path Encoding}
\label{sec:ghrr-binding}
A GHRR hypervector is a block vector $\mathbf{H} = [A_1;\dots;A_D]$ with
$A_j \in \mathrm{U}(m)$. Given two hypervectors
$\mathbf{X} = [X_1;\dots;X_D]$ and $\mathbf{Y} = [Y_1;\dots;Y_D]$, we define
the block-wise binding operator $\bbind$ and the encoding of a length-$\ell$
relation path $z = (r_1,\dots,r_\ell)$ by:
\begin{equation}
\mathbf{v}_z
\;=\;
\mathbf{v}_{r_1} \bbind \mathbf{v}_{r_2} \bbind \cdots \bbind \mathbf{v}_{r_\ell},
\qquad
\mathbf{X} \bbind \mathbf{Y} = [X_1 Y_1;\dots;X_D Y_D],
\label{eq:ghrr-path}
\end{equation}
followed by block-wise normalization to unit norm.
Binding is applied left-to-right along the path, and because the matrix
multiplication is non-commutative ($X_j Y_j \neq Y_j X_j$), the encoding
preserves the order and directionality of relations, which are critical for
multi-hop KG reasoning.

\paragraph{Properties and choice of binding operator.}
Although {\ourmethod} only uses forward binding for retrieval, GHRR also supports
approximate unbinding: for $Z_j = X_j Y_j$ with unitary blocks, we have
$X_j \approx Z_j Y_j^\ast$ and $Y_j \approx X_j^\ast Z_j$.
This property enables inspection of the contribution of individual relations in
a composed path and underpins our path-level rationales.

Classical HDC bindings (XOR, element-wise multiplication, circular convolution)
are \emph{commutative}, which collapses $r_1{\to}r_2$ and $r_2{\to}r_1$ to
similar codes and hurts directional reasoning.
GHRR is non-commutative, invertible at the block level, and offers higher
representational capacity via unitary blocks, leading to better discrimination
between paths of the same multiset but different order.
We empirically validate this choice in the ablation study
(\Cref{tab:ablate-binding}), where GHRR consistently outperforms commutative
bindings. Additional background on binding operations is provided in
\Cref{app:binding}.

\subsection{Query \& Candidate Path Construction}
\label{sec:path_construction}
We obtain a query plan $z_q$ via schema-based enumeration on the
relation-schema graph (depth $\le L_{\max}$).
In all main experiments reported in \Cref{sec:exp}, this planning stage
is purely symbolic: we do \emph{not} invoke any LLM beyond the single final
reasoning call in \Cref{sec:reasoning}.
The query hypervector $\mathbf{v}_q$ is then constructed from the selected plan
$z_q$ using the plan-based encoding described above
(Eq.~\eqref{eq:ghrr-path} or, for text projection, Eq.~\eqref{eq:query-hv}).

Given a query plan, candidate paths $\mathcal{Z}$ are instantiated from the KG
either by matching plan templates to existing edges or by a constrained BFS
with beam width $B$, both of which yield symbolic paths.
These paths are then deterministically encoded into hypervectors and scored by
our HDC module (Sec.~\ref{sec:hd-retrieval}).
An optional lightweight prompt-based refinement of schema plans is described in
the appendix as an extension; it is \emph{not} used in our main experiments and
does not change the single-call nature of the system.

\subsection{HD Retrieval: Blockwise Similarity and Top-$K$}
\label{sec:hd-retrieval}
Let $\langle A,B\rangle_F := \mathrm{tr}(A^\ast B)$ be the Frobenius inner
product. Given two GHRR hypervectors
$\mathbf{X}=[X_j]_{j=1}^D$ and $\mathbf{Y}=[Y_j]_{j=1}^D$, we define the
\emph{blockwise cosine similarity}
\begin{equation}
\mathrm{sim}(\mathbf{X},\mathbf{Y})
=\frac{1}{D}\sum_{j=1}^{D}\,
\frac{\Re\,\langle X_j,\,Y_j\rangle_F}{\|X_j\|_F\,\|Y_j\|_F}.
\label{eq:block-cos}
\end{equation}
For each candidate $z\in\mathcal{Z}$ we compute
$\mathrm{sim}(\mathbf{v}_q,\mathbf{v}_z)$ and (optionally) apply a calibrated
score
\begin{equation}
s(z) \;=\; \mathrm{sim}(\mathbf{v}_q,\mathbf{v}_z)\;+\;\alpha\,\mathrm{IDF}(z)\;-\;\beta\,\lambda^{|z|},
\label{eq:idf_hyper}
\end{equation}
where $\mathrm{IDF}(z)$ is a simple inverse-frequency weight on relation
schemas. Let $\text{schema}(z)$ denote the relation schema of path $z$ and
$\mathrm{freq}(\text{schema}(z))$ be the number of training questions whose
candidate sets contain at least one path with the same schema. With
$N_{\text{train}}$ the total number of training questions, we define:
\begin{equation}
    \mathrm{IDF}(z)
    = \log\!\left(1 + \frac{N_{\text{train}}}{1 + \mathrm{freq}(\text{schema}(z))}\right).
\end{equation}
Thus, frequent schemas (large $\mathrm{freq}(\text{schema}(z))$) receive a
smaller bonus, while rare schemas receive a larger one.
All similarity scores and calibrated scores can be computed fully in parallel
over candidates, with overall cost $\mathcal{O}(|\mathcal{Z}|d)$.

\subsection{One-shot Reasoning with Retrieved Paths}
\label{sec:reasoning}
Putting the pieces together, {\ourmethod} turns a question into
(i) a schema-level plan $z_q$,
(ii) a set of candidate paths ranked in hypervector space, and
(iii) a single LLM call that adjudicates among the top-ranked candidates.

We linearize the Top-$K$ paths into concise natural-language statements and
issue a \emph{single} LLM call with a minimal, citation-style prompt
(see~\Cref{tab:prompt-template} from~\Cref{app:prompt}).
The prompt lists the question and the numbered paths, and requires the model to
return a short answer, the index(es) of supporting path(s), and a 1-2 sentence
rationale.
This one-shot format constrains reasoning to the provided evidence, resolves
near-ties and direction errors, and keeps LLM usage minimal.

\subsection{Theoretical \& Complexity Analysis}
\begin{wraptable}{r}{0.47\textwidth} 
\vspace{-3em}
 \centering
\small
\renewcommand{\arraystretch}{1.0}
\resizebox{0.47\textwidth}{!}{
\begin{tabular}{l||ccc}
\Xhline{1.2pt}
\rowcolor{tablehead!20} \textbf{Method} & \makecell{\textbf{Candidate}\\\textbf{Path Gen.}} & \textbf{Scoring} & \textbf{Reasoning} \\
\hline\hline
\rowcolor{gray!10}StructGPT~\citep{jiang2023structgpt}  & \textcolor{green(pigment)}{\Checkmark} & \textcolor{green(pigment)}{\Checkmark} & \textcolor{green(pigment)}{\Checkmark} \\
FiDeLiS~\citep{sui2024fidelis}        & \textcolor{green(pigment)}{\Checkmark} & \textcolor{darksalmon}{\XSolidBrush} & \textcolor{green(pigment)}{\Checkmark} \\
\rowcolor{gray!10}ToG~\citep{sun2023think}              & \textcolor{green(pigment)}{\Checkmark} & \textcolor{green(pigment)}{\Checkmark} & \textcolor{green(pigment)}{\Checkmark} \\
GoG~\citep{xu2024generate}            & \textcolor{green(pigment)}{\Checkmark} & \textcolor{green(pigment)}{\Checkmark} & \textcolor{green(pigment)}{\Checkmark} \\
\rowcolor{gray!10}KG-Agent~\citep{jiang2024kg}          & \textcolor{green(pigment)}{\Checkmark} & \textcolor{green(pigment)}{\Checkmark} & \textcolor{green(pigment)}{\Checkmark} \\
RoG~\citep{luo2023reasoning}          & \textcolor{green(pigment)}{\Checkmark} & \textcolor{darksalmon}{\XSolidBrush} & \textcolor{green(pigment)}{\Checkmark} \\
\rowcolor[HTML]{D7F6FF}\ourmethod & \textcolor{darksalmon}{\XSolidBrush} & \textcolor{darksalmon}{\XSolidBrush} & \textcolor{green(pigment)}{\Checkmark} ($1$ call) \\
\Xhline{1.2pt}
\end{tabular}}
\vspace{-0.3em}
\caption{
LLM usage across pipeline stages. A checkmark indicates that the method uses an LLM in that stage.
\emph{Candidate Path Gen.}: using an LLM to propose or expand relation paths;
\emph{Scoring}: using an LLM to score or rank candidates (non-LLM similarity or graph heuristics count as ``no'');
\emph{Reasoning}: using an LLM to produce the final answer from the retrieved paths.
\ourmethod uses a single LLM call only in the final reasoning stage.
}
\label{tab:llm-usage}
\vspace{-1.5em}
\end{wraptable}
We briefly characterize the behavior of random GHRR hypervectors and the
computational cost of {\ourmethod}.

\begin{proposition}[Near-orthogonality and distractor bound]
\label{prop:near-orth}
Let $\{\mathbf{v}_r\}$ be i.i.d.\ GHRR hypervectors with zero-mean, unit
Frobenius-norm blocks. 
For a query path $z_q$ and any distractor $z\neq z_q$ encoded via
non-commutative binding, the cosine similarity
$X=\mathrm{sim}(\mathbf{v}_{z_q},\mathbf{v}_{z})$ (\Cref{eq:block-cos})
satisfies, for any $\epsilon>0$,
\begin{equation}
\Pr\!\left(|X|\ge \epsilon\right) \le 2\exp\!\left(-c\, d\, \epsilon^2\right),
\end{equation}
for an absolute constant $c>0$ depending only on the sub-Gaussian proxy of
entries.
\end{proposition}

\begin{proof}[Proof sketch]
Each block inner product $\langle X_j,Y_j\rangle_F$ is a sum of products of
independent sub-Gaussian variables (closed under products for the bounded/phase
variables used by GHRR). After normalization, the average in
\Cref{eq:block-cos} is a mean-zero sub-Gaussian average over $d$ degrees of
freedom, so a standard Bernstein/Hoeffding tail bound applies. Details are
deferred to \Cref{app:proofs}.
\end{proof}

\begin{corollary}[Capacity with union bound]
\label{cor:capacity}
Let $\mathcal{M}$ be a collection of $M$ distractor paths scored against a
fixed query. With probability at least $1-\delta$,
\begin{equation}
\max_{z\in\mathcal{M}} \mathrm{sim}(\mathbf{v}_{z_q},\mathbf{v}_{z}) \le \epsilon
\quad\text{whenever}\quad
d \;\ge\; \frac{1}{c\,\epsilon^2}\,\log\!\frac{2M}{\delta}.
\end{equation}
\end{corollary}

Thus, the probability of a false match under random hypervectors decays
exponentially with the dimension $d$, and the required dimension scales as
$d=\mathcal{O}(\epsilon^{-2}\log M)$ for a target error tolerance $\epsilon$.

\paragraph{Complexity comparison with neural retrievers.}
Let $N$ be the number of candidates, $d$ the embedding dimension, and $L$ the
number of encoder layers used by a neural retriever.
A typical neural encoder (e.g., Transformer-based path encoder as in RoG)
incurs $\mathcal{O}(N L d^2)$ cost for encoding and scoring.
In contrast, \ourmethod forms each path vector by $|z|\!-\!1$ block
multiplications plus one similarity in \Cref{eq:block-cos}, i.e.,
$\mathcal{O}(|z|d)+\mathcal{O}(d)$ per candidate, giving a total of
$\mathcal{O}(Nd)$ and an $\mathcal{O}(L d)$-fold reduction in leading order.

Beyond the $\mathcal{O}(Nd)$ vs.\ $\mathcal{O}(N L d^2)$ compute gap,
end-to-end latency is dominated by the number of LLM calls.
\Cref{tab:llm-usage} contrasts pipeline stages across methods: unlike prior
agents that query an LLM for candidate path generation and sometimes for
scoring, \ourmethod defers a single LLM call to the final reasoning step.
This design reduces both latency and API cost; empirical results in
\Cref{sec:efficiency} confirm the shorter response times.

\section{Experiments}\label{sec:exp}
We evaluate \ourmethod against state-of-the-art baselines on reasoning accuracy,
measure efficiency with a focus on end-to-end latency, and conduct module-wise
ablations followed by illustrative case studies.

\subsection{Datasets, Baselines, and Setup}
We evaluate on three standard multi-hop KGQA benchmarks:
\textbf{WebQuestionsSP (WebQSP)}~\citep{yih2016value},
\textbf{Complex WebQuestions (CWQ)}~\citep{talmor2018web},
and \textbf{GrailQA}~\citep{gu2021beyond}, all grounded in
\textbf{Freebase}~\citep{bollacker2008freebase}.
These datasets span increasing reasoning complexity (roughly 2-4 hops):
WebQSP features simpler single-turn queries, CWQ adds compositional and
constraint-based questions, and GrailQA stresses generalization across i.i.d.,
compositional, and zero-shot splits.
Our study is therefore scoped to Freebase-style KGQA, extending {\ourmethod} to
domain-specific KGs are left for future work.

We compare against four families of methods:
\textbf{embedding-based}, \textbf{retrieval-augmented},
\textbf{pure LLMs} (no external KG), and \textbf{LLM+KG hybrids}.
All results are reported on dev (IID) splits under a unified Freebase evaluation
protocol using the official \textit{Hits@1} and \textit{F1} scripts, so that numbers are directly comparable across systems.
Unless otherwise noted, \ourmethod uses the schema-based planner from
\Cref{sec:path_construction} (no LLM calls during planning), the plan-based
query hypervector in \Cref{eq:ghrr-path}, and a single LLM adjudication call as
described in \Cref{sec:reasoning}, with all LLMs and sentence encoders used
\emph{off the shelf} (no fine-tuning).
Detailed dataset statistics, baseline lists, model choices, and additional
training-free hyperparameters are provided in
\Cref{appendix:datasets,appendix:baselines,app:setup}.

\subsection{Reasoning Performance Comparison}
\begin{table}[t]
\centering
\small
        \setlength{\tabcolsep}{4pt}
        \renewcommand{\arraystretch}{1.0}
\begin{tabular}{ll||ccIccIcc}
\Xhline{1.2pt}
\rowcolor{tablehead!20}  &&
\multicolumn{2}{cI}{\textbf{WebQSP}} & 
\multicolumn{2}{cI}{\textbf{CWQ}} & 
\multicolumn{2}{c}{\textbf{GrailQA (F1)}}\\
\cline{3-4}\cline{5-6}\cline{7-8}
 \rowcolor{tablehead!20} \multirow{-2}{*}{\textbf{Type}}&\multirow{-2}{*}{\textbf{Methods}} & {\textbf{Hits@$1$}} & {\textbf{F1}} & {\textbf{Hits@$1$}} & {\textbf{F1}} & {\textbf{Overall}} & {\textbf{IID}}   \\
\hline\hline
\rowcolor{gray!10}&KV-Mem \citep{miller2016key}           & $46.7$ & $34.5$ & $18.4$ & $15.7$ &  $-$  &  $-$  \\
&EmbedKGQA \citep{saxena2020improving}         & $66.6$ &  $-$  & $45.9$ &  $-$  &  $-$  &  $-$  \\
\rowcolor{gray!10}&NSM \citep{he2021improving}               & $68.7$ & $62.8$ & $47.6$ & $42.4$ &  $-$  &  $-$  \\
\multirow{-4}{*}{Embedding}& TransferNet \citep{shi2021transfernet}       & $71.4$ &  $-$  & $48.6$ &  $-$  &  $-$  &  $-$  \\
\hline
\rowcolor{gray!10}  &GraftNet \citep{sun2018open}         & $66.4$ & $60.4$ & $36.8$ & $32.7$ &  $-$  &  $-$  \\
& SR+NSM \citep{zhang2022subgraph}           & $68.9$ & $64.1$ & $50.2$ & $47.1$ &  $-$  &  $-$  \\
\rowcolor{gray!10}& SR+NSM+E2E \citep{zhang2022subgraph} & $69.5$ & $64.1$ & $49.3$ & $46.3$ &  $-$ & $-$ \\
\rowcolor{gray!10}\multirow{-4}{*}{Retrieval}  &UniKGQA \citep{jiang2022unikgqa}           & $77.2$ & $72.2$ & $51.2$ & $49.1$ &  $-$  &  $-$  \\
\hline
& ChatGPT \citep{ouyang2022training}         & $67.4$ & $59.3$ & $47.5$ & $43.2$ & $25.3$ & $19.6$ \\
\rowcolor{gray!10}& Davinci\text{-}003 \citep{ouyang2022training} & $70.8$ & $63.9$ & $51.4$ & $47.6$ & $30.1$ & $23.5$\\
\multirow{-3}{*}{Pure LLMs} & GPT\text{-}4 \citep{achiam2023gpt}        & $73.2$ & $62.3$ & $55.6$ & $49.9$ & $31.7$ & $25.0$ \\
\hline
\rowcolor{gray!10}&StructGPT \citep{jiang2023structgpt} & $72.6$ & $63.7$ & $54.3$ & $49.6$ & $54.6$ & $70.4$ \\
& ROG \citep{luo2023reasoning}               & $\underline{85.7}$ & $70.8$ & $62.6$ & $56.2$ &  $-$   &  $-$  \\
\rowcolor{gray!10}& Think-on-Graph \citep{sun2023think}    & $81.8$ & $76.0$ & $68.5$ & $60.2$ &  $-$  &  $-$  \\
& GoG \citep{xu2024generate}      & $84.4$ &  $-$  & $\mathbf{75.2}$ &  $-$  &  $-$  &  $-$  \\
\rowcolor{gray!10}&KG-Agent  \citep{jiang2024kg}       & $83.3$ & $\mathbf{81.0}$ & $\underline{72.2}$ & $\mathbf{69.8}$ & $\underline{86.1}$ & $\underline{92.0}$ \\
& FiDeLiS \citep{sui2024fidelis}  & $84.4$ & $78.3$ & $71.5$ & $64.3$ &  $-$  &  $-$  \\
\rowcolor[HTML]{D7F6FF}\multirow{-7}{*}{LLMs + KG}&\textbf{\ourmethod} & $\mathbf{86.2}$ & $\underline{78.6}$ & $71.5$ & $\underline{65.8}$ & $\mathbf{86.7}$ & $\mathbf{92.4}$\\
\Xhline{1.2pt}
\end{tabular}
\caption{\textbf{Comparison on Freebase-based KGQA.} Our method {\ourmethod} follows exactly the same protocol. ``--'' indicates that the metric was \emph{not reported by the original papers under the Freebase+official-script setting}. We bold the best and underline the second-best score for each metric/column.}
\label{tab:main_results}
\end{table}
We evaluate under a unified Freebase protocol with the official
\textit{Hits@1}/\textit{F1} scripts on WebQSP, CWQ, and GrailQA (dev, IID);
results are in \Cref{tab:main_results}.
Baselines cover classic KGQA (embedding/retrieval), recent LLM+KG systems, and
pure LLMs (no KG grounding).
Our \ourmethod uses hyperdimensional scoring with GHRR, Top-$K$ pruning, and a
\emph{single} LLM adjudication step.

Key observations are as follows.
\textbf{Obs.\ding{182}} \textbf{SOTA on WebQSP/GrailQA; competitive on CWQ.}
\ourmethod attains best WebQSP \textit{Hits@1} ($86.2$) and best GrailQA
\textit{F1} (Overall/IID $86.7/92.4$), while staying strong on CWQ
(\textit{Hits@1} $71.5$, \textit{F1} $65.8$), close to the top LLM+KG systems
(e.g., GoG $75.2$ \textit{Hits@1}; KG-Agent $69.8$ \textit{F1}).
\textbf{Obs.\ding{183}} \textbf{One-shot adjudication rivals multi-step agents.}
Compared to RoG ($\sim$12 calls) and Think-on-Graph/GoG/KG-Agent ($3$–$8$
calls), \ourmethod matches or exceeds accuracy on WebQSP/GrailQA and remains
competitive on CWQ with just \emph{one} LLM call, which reduces error
compounding and focuses the LLM on a high-quality shortlist.
\textbf{Obs.\ding{184}} \textbf{Pure LLMs lag without KG grounding.}
Zero/few-shot GPT-4 or ChatGPT underperform LLM+KG systems, e.g., on CWQ
GPT-4 \textit{Hits@1} $55.6$ vs.\ \ourmethod $71.5$.
\textbf{Obs.\ding{185}} \textbf{Classic embedding/retrieval trails modern LLM+KG.} KV-Mem, NSM, SR+NSM rank subgraphs well but lack a flexible language component
for composing multi-hop constraints, yielding consistently lower scores.

\noindent\textbf{Candidate enumeration strategy.}
Before turning to efficiency (\Cref{sec:efficiency}), we briefly clarify how candidate paths are enumerated, since this affects both accuracy and cost.
In our current implementation, we use a deterministic BFS-style enumeration of
relation paths, controlled by the maximum depth $L_{\max}$ and beam width $B$.
This choice is (i) simple and efficient, (ii) guarantees coverage of all
type-consistent paths up to length $L_{\max}$ under clear complexity bounds,
and (iii) makes it easy to compare against prior KGQA baselines that also rely
on BFS-like expansion.
In practice, we choose $L_{\max}$ and $B$ to achieve high coverage of gold
answer paths while keeping candidate set sizes comparable to RoG and KG-Agent.
More sophisticated, adaptive enumeration strategies (e.g., letting the HDC scores or the LLM guide, which relations to expand next) are an interesting
extension, but orthogonal to our core contribution.

\subsection{Efficiency and Cost Analysis}\label{sec:efficiency}
\begin{figure}[t]
	\centering
	\includegraphics[width=0.32\textwidth]{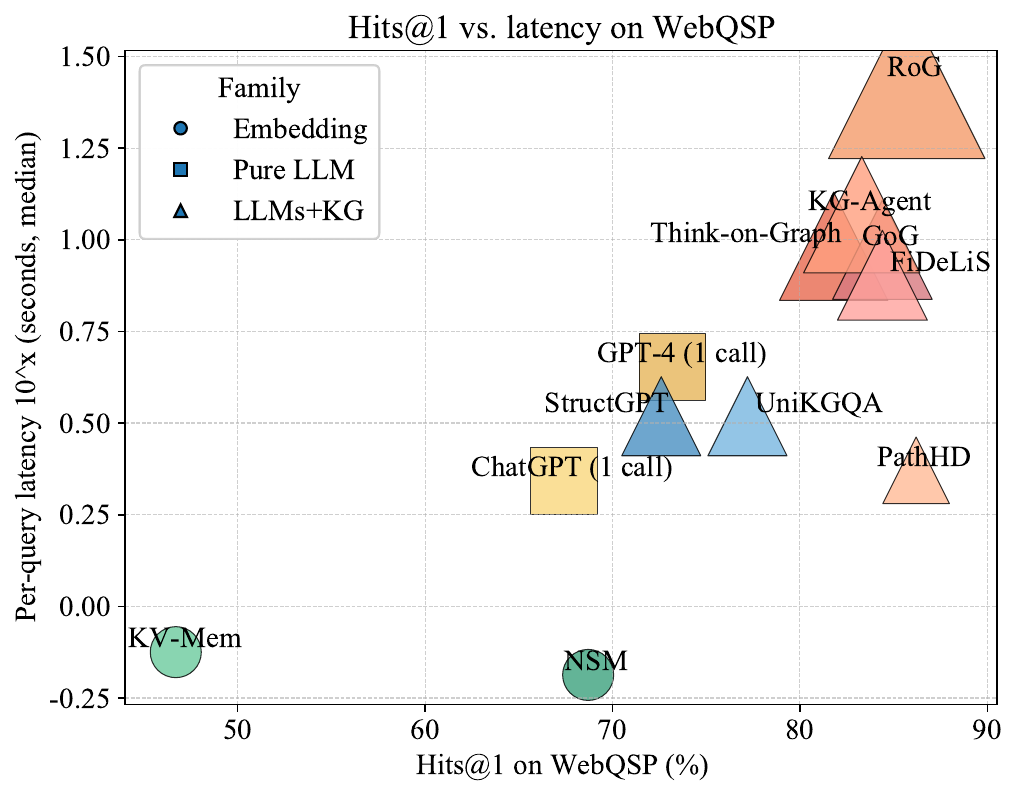}
\includegraphics[width=0.32\textwidth]{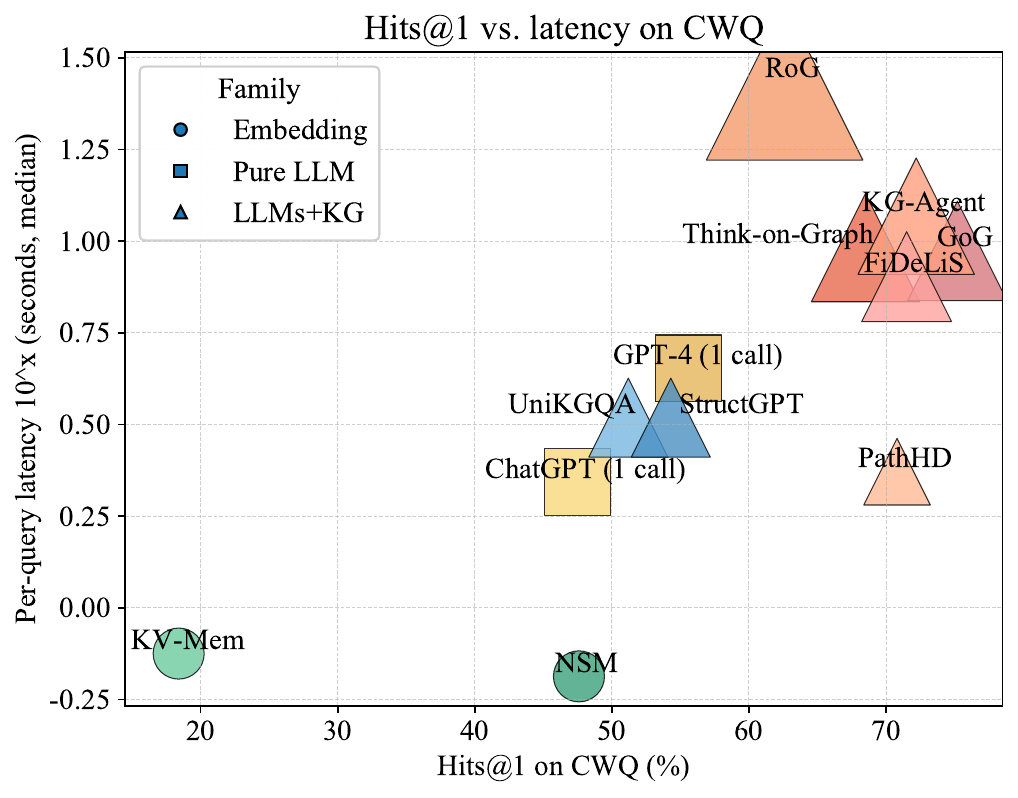}
        \includegraphics[width=0.32\textwidth]{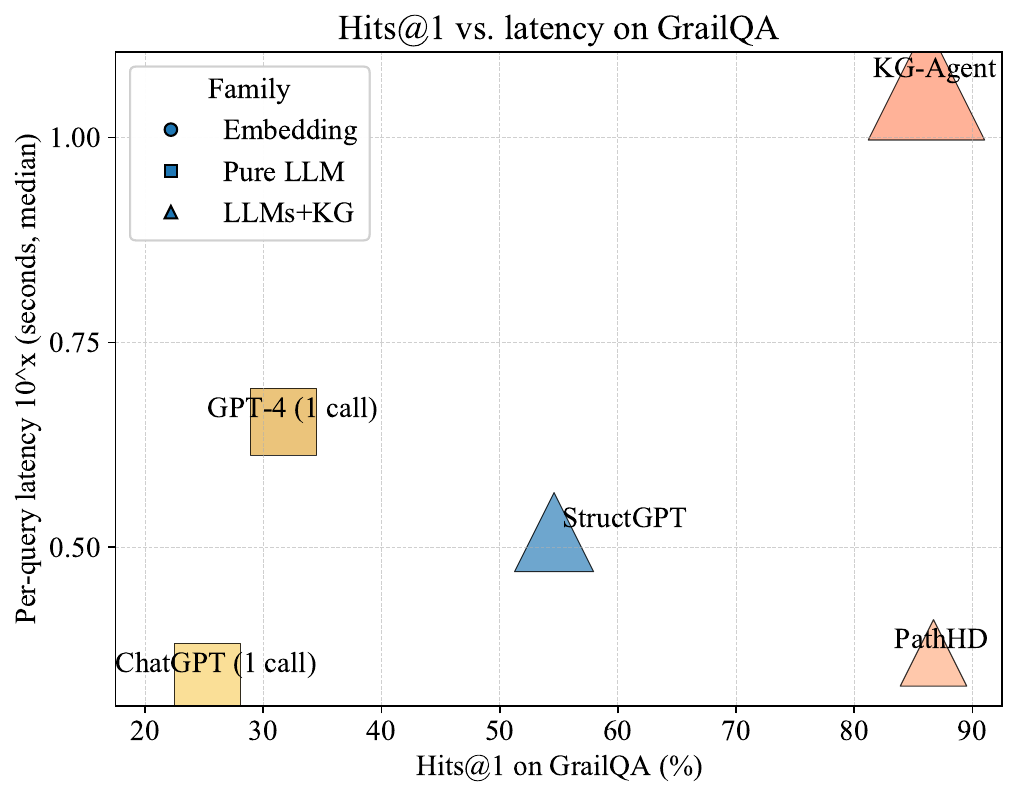}
        \vspace{-0.5em}
\caption{\textbf{Visualization of performance and latency.} 
The x-axis is Hits@$1$ (\%), the y-axis is per-query latency in seconds (median, log scale). 
Bubble size indicates the average number of LLM calls; marker shape denotes the method family. 
\ourmethod gives strong accuracy with lower latency than multi-call LLMs+KG baselines.}
\label{fig:bubble}
\vspace{-0.5em}
\end{figure}
We assess end-to-end cost via a \emph{Hits@1--latency} bubble plot
(\Cref{fig:bubble}) and a lollipop latency chart (\Cref{fig:lollipop}).
In \Cref{fig:bubble}, $x$ = \textit{Hits@1}, $y$ = median per-query latency
(log-scale); bubble size = average \#LLM calls; marker shape = method family.
All latencies are measured under a common hardware and LLM-backend setup to
enable fair relative comparison.
Latencies in \Cref{fig:lollipop} follow a shared protocol: in our
implementation, each LLM call takes on the order of a few seconds, whereas
non-LLM vector/graph operations are typically within $0.3$--$0.8$s.
\ourmethod uses vector-space scoring with Top-$K$ pruning and a \emph{single}
LLM decision; RoG uses beam search ($B{=}3$, depth $\le$ dataset hops).
A factor breakdown (\#calls, depth $d$, beam $b$, tools) appears in
\Cref{tab:analytic_full} (\Cref{app:ana_eff}).

Key observations are:
\textbf{Obs.\ding{182}} \textbf{Near-Pareto across datasets.}
With comparable accuracy to multi-call LLM+KG systems
(Think-on-Graph/GoG/KG-Agent), \ourmethod achieves markedly lower latency due
to its single-call design and compact post-pruning candidate set.
\textbf{Obs.\ding{183}} \textbf{Latency is dominated by \#LLM calls.}
Methods with 3–8 calls (agent loops) or $\approx d\times b$ calls (beam search)
sit higher in \Cref{fig:bubble} and show longer whiskers in
\Cref{fig:lollipop}; \ourmethod avoids intermediate planning/scoring overhead.
\textbf{Obs.\ding{184}} \textbf{Moderate pruning improves cost--accuracy.}
Shrinking the pool before adjudication lowers latency without hurting
\textit{Hits@1}, especially on CWQ, where paths are longer.
\textbf{Obs.\ding{185}} \textbf{Pure LLMs are fast but underpowered.} Single-call GPT-4/ChatGPT has similar latency to our final decision yet
notably lower accuracy, underscoring the importance of structured retrieval and
path scoring.

\subsection{Ablation Study}
We analyze the contribution of each module/operation in \ourmethod. 
Our operation study covers: (1) \textit{Path composition operator}, (2) \textit{Single-LLM adjudicator}, and (3) \textit{Top-$K$ pruning}.
\vspace{-0.5em}

\begin{wraptable}{r}{0.5\textwidth} 
\vspace{-1em}
\centering
\small
\renewcommand{\arraystretch}{1.0}
\resizebox{0.5\textwidth}{!}{
\begin{tabular}{l||cc}
\Xhline{1.2pt}
\rowcolor{tablehead!20}\textbf{Operator} & \textbf{WebQSP} & \textbf{CWQ} \\
\hline\hline
\rowcolor{gray!10}XOR / bipolar product & 83.9 & 68.8 \\
 Element-wise product (Real-valued)        & 84.4 & 69.2 \\
\rowcolor{gray!10}Comm.\ bind      & 84.7 & 69.6 \\
FHRR                     & 84.9 & 70.0 \\
\rowcolor{gray!10}HRR        & 85.1 & 70.2 \\
\textbf{GHRR}            & \textbf{86.2} & \textbf{71.5} \\
\Xhline{1.2pt}
\end{tabular}}
\vspace{-0.5em}
\caption{\textbf{Effect of the path–composition operator.} 
GHRR yields the best performance.}
\label{tab:ablate-binding}
\vspace{-0.5em}
\end{wraptable}
\paragraph{Which path–composition operator works best?}
We isolate relation binding by fixing retrieval, scoring, pruning, and the single LLM step, and \emph{only} swapping the encoder’s path–composition operator. We compare six options (defs. in \Cref{app:binding}): (i) \textit{XOR/bipolar} and (ii) real-valued element-wise products, both fully \emph{commutative}; (iii) a stronger \emph{commutative} mix of binary/bipolar; (iv) \textit{FHRR} (phasors) and (v) \textit{HRR} (circular convolution), efficient yet effectively commutative; and (vi) our \textit{block-diagonal GHRR} with unitary blocks, \emph{non-commutative} and order-preserving. Paths of length $1$–$4$ use identical dimension/normalization. As in \Cref{tab:ablate-binding}, commutative binds lag, HRR/FHRR give modest gains, and \textbf{GHRR} yields the best \textit{Hits@1} on WebQSP and CWQ by reliably separating \textit{founded\_by}$\rightarrow$\textit{CEO\_of} from its reverse.
\vspace{-0.5em}

\begin{wraptable}{r}{0.44\textwidth} 
\vspace{-1.2em}
\centering
\small
\renewcommand{\arraystretch}{1.0}
\begin{tabular}{l||cc}
\Xhline{1.2pt}
\rowcolor{tablehead!20}\textbf{Final step} & \textbf{WebQSP} & \textbf{CWQ} \\
\hline\hline
\rowcolor{gray!10}Vector-only            & 85.4 & 70.8 \\
\textbf{Vector $\rightarrow$ 1$\times$LLM} & \textbf{86.2} & \textbf{71.5} \\
\Xhline{1.2pt}
\end{tabular}
\vspace{-0.5em}
\caption{\textbf{Ablation on the final decision maker}. Passing pruned candidates and scores
to a single LLM for adjudication yields consistent gains over vector-only selection.}
\label{tab:ablate-final-llm}
\vspace{-1em}
\end{wraptable}
\vspace{-0.5em}
\paragraph{Do we need a final single LLM adjudicator?}
We test whether a lightweight LLM judgment helps beyond pure vector scoring. \textit{Vector-only} selects the top path by cosine similarity; \textit{Vector $\rightarrow$ 1$\times$LLM} instead forwards the pruned top-$K$ paths (with scores and end entities) to a single LLM using a short fixed template (no tools/planning) to choose the answer \emph{without} long chains of thought. As shown in \Cref{tab:ablate-final-llm}, \textbf{Vector $\rightarrow$ 1$\times$LLM} consistently outperforms \textit{Vector-only} on both datasets, especially when the top two paths are near-tied or a high-scoring path has a subtle type mismatch; a single adjudication pass resolves such cases at negligible extra cost.

\begin{wraptable}{r}{0.45\textwidth} 
\vspace{-1em}
\centering
\small
\setlength{\tabcolsep}{2pt}
\begin{tabular}{l||cccc}
\Xhline{1.2pt}
\rowcolor{tablehead!20}\textbf{Pruning} & \makecell{\textbf{Hits@$1$}\\\textbf{(WebQSP)}} & \textbf{Lat.} & \makecell{\textbf{Hits@$1$}\\\textbf{(CWQ)}} & \textbf{Lat.} \\
\hline\hline
\rowcolor{gray!10}No-prune & 85.8 & 2.42s & 70.7 & 2.45s \\
$K{=}2$                     & 86.0 & \underline{1.98s} & 71.2 & \underline{2.00s} \\
\rowcolor{gray!10}\textbf{$K{=}3$} 
                           & \textbf{86.2} & \textbf{1.92s} & \textbf{71.5} & \textbf{1.94s} \\
$K{=}5$                     & \underline{86.1} & 2.05s & \underline{71.4} & 2.06s \\
\Xhline{1.2pt}
\end{tabular}
\vspace{-0.5em}
\caption{\textbf{Impact of top-$K$ pruning before the final LLM.} Small sets (K=2–3) retain or slightly improve accuracy while reducing latency. We adopt \textbf{$K{=}3$} by default.}
\label{tab:ablate-topk}
\vspace{-0.5em}
\end{wraptable}
\paragraph{What is the effect of top-$K$ pruning before the final step?}
Finally, we study how many candidates should be kept for the last decision. We vary the number of paths passed to the final LLM among $K\!\in\!\{2,3,5\}$ and also include a \textit{No-prune} variant that sends all retrieved paths. Retrieval and scoring are fixed; latency is the median per query (lower is better). As shown in \Cref{tab:ablate-topk}, \textbf{$K{=}3$} achieves the best Hits@$1$ on both WebQSP and CWQ with the lowest latency, while $K{=}2$ is a close second and yields the largest latency drop. In contrast, \textit{No-prune} maintains maximal recall but increases latency and often introduces near-duplicate/noisy paths that can blur the final decision. We therefore adopt \textbf{$K{=}3$} as the default.

\begin{figure*}[!t]
\vspace{-0.6em}
\centering
\includegraphics[width=\linewidth]{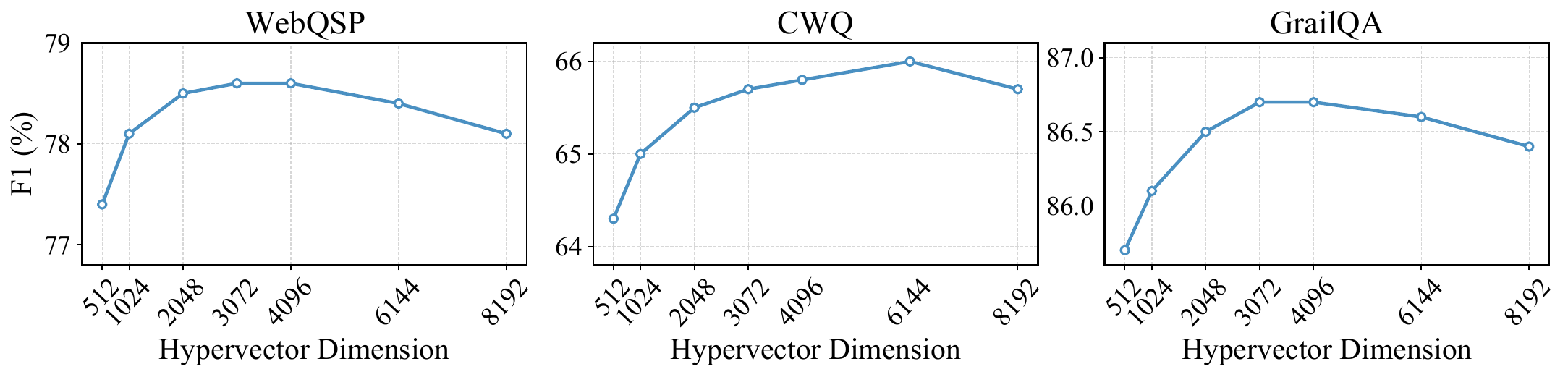}
\vspace{-1.9em}
\caption{\textbf{Hypervector dimension study.} Each panel reports F1 (\%) of \ourmethod on WebQSP, CWQ, and GrailQA as a function of the hypervector dimension. Overall, performance rises from $512$ to the mid-range and then tapers off: WebQSP and GrailQA peak around 3k–4k, while CWQ prefers a slightly larger size (6k), after which F1 decreases mildly.}
\label{fig:attack}
\vspace{-1em}
\end{figure*}

\subsection{Case Study}\label{sec:case}
To better understand how our model performs step-by-step reasoning, we present two representative cases from the WebQSP dataset in~\Cref{table:webqsp-case}. These cases highlight the effects of candidate path pruning and the contribution of LLM-based adjudication in improving answer accuracy.
\noindent\textbf{Case 1}:
Top-$K$ pruning preserves paths aligned with both \texttt{film.film.music} and actor cues; the vector-only scorer already picks the correct path, and a single LLM adjudication confirms \textit{Valentine's Day}, illustrating that pruning reduces cost while retaining high-coverage candidates.
\noindent\textbf{Case 2}:
A vector-only top path (\texttt{film.film.edited\_by}) misses the actor constraint and yields a false positive, but adjudication over the pruned set, now including \texttt{performance.actor}, corrects to \textit{The Perks of Being a Wallflower}, showing that LLM adjudication resolves compositional constraints beyond static similarity.

\begin{table*}[t]
\centering
\small
\setlength{\tabcolsep}{4pt}
\renewcommand{\arraystretch}{1.0}
\newcommand{\code}[1]{\texttt{\footnotesize #1}}
\newcommand{\good}{\textcolor{ForestGreen}{\ding{51}}}  %
\newcommand{\bad}{\textcolor{BrickRed}{\ding{55}}}      %
\vspace{-1em}
\begin{tabularx}{\textwidth}{lX}
\Xhline{1.2pt}
\multicolumn{2}{l}{\textbf{Case 1:} \emph{which movies featured Taylor Swift and music by John Debney}}\\
\hline
\multirow{4}{*}{\textbf{Top-4 candidates}} & 1) \code{film.film.music} (0.2567) \\
& 2) \code{person.nationality $\rightarrow$ film.film.country} (0.2524) \\
& 3) \code{performance.actor $\rightarrow$ performance.film} (0.2479) \\
& 4) \code{people.person.languages $\rightarrow$ film.film.language} (0.2430) \\
\hline
\multirow{3}{*}{\makecell{\textbf{Top-$K$ after pruning}\\ (K{=}3)}}& \code{film.film.music}  \\
&\code{person.nationality $\rightarrow$ film.film.country} \\
& \code{performance.actor $\rightarrow$ performance.film} \\
\hline
\multirow{2}{*}{\textbf{Vector-only (no LLM)}} &
Pick \code{film.film.music} \good\ — directly targets the composer-to-film mapping; relevant for filtering by music. \\
\hline
\multirow{3}{*}{\textbf{1$\times$LLM adjudication}} &
\emph{Rationale:} ``To find films with both Taylor Swift and music by John Debney, use actor-to-film and music-to-film relations. The chosen path targets the latter directly.'' \\
\hline
\textbf{Final Answer / GT} & \textbf{Valentine’s Day (predict)} / \textbf{Valentine’s Day} \good \\
\hline
\hline
\multicolumn{2}{l}{\textbf{Case 2 :} \emph{in which movies does Logan Lerman act in that was edited by Mary Jo Markey}}\\
\hline
\multirow{5}{*}{\textbf{Top-4 candidates}} &1) \code{film.film.edited\_by} (0.2548)\\
& 2) \code{person.nationality $\rightarrow$ film.film.country} (0.2527)\\
& 3) \code{performance.actor $\rightarrow$ performance.film} (0.2505) \\
& 4) \code{award.award\_winner.awards\_won} $\rightarrow$  \\
& \code{award.award\_honor.honored\_for} (0.2420) \\
\hline
\multirow{3}{*}{\makecell{\textbf{Top-$K$ after pruning}\\ (K{=}3)}}& \code{film.film.edited\_by}  \\
&\code{person.nationality $\rightarrow$ film.film.country} \\
& \code{performance.actor $\rightarrow$ performance.film} \\
\hline
\multirow{2}{*}{\textbf{Vector-only (no LLM)}} &
Pick \code{film.film.edited\_by} \bad\ — identifies edited films, but lacks actor constraint; leads to false positives. \\
\hline
\hline
\multirow{3}{*}{\textbf{1$\times$LLM adjudication}} &
\emph{Rationale:} ``The question requires jointly filtering for actor and editor. While \code{film.edited\_by} is relevant, combining it with \code{performance.actor} improves precision by ensuring Logan Lerman is in the cast.'' \\
\hline
\textbf{Final Answer / GT} & \textbf{Perks of Being a Wallflower (predict)} / \textbf{Perks of Being a Wallflower} \good \\
\Xhline{1.2pt}
\end{tabularx}
\caption{\textbf{Case studies on multi-hop reasoning over WebQSP.} Top-$K$ pruning is applied before invoking LLM, reducing cost while retaining plausible candidates.}
\label{table:webqsp-case}
\end{table*}

\noindent\textbf{Discussion.}
As {\ourmethod} operates in a single-call, fixed-candidate regime, its performance
ultimately depends on (i) the Top-$K$ retrieved paths covering at least one
valid reasoning chain and (ii) the adjudication LLM correctly ranking these
candidates. In practice, we mitigate this by using a relatively generous $K$
(e.g., $K=3$) and beam widths that yield high coverage of gold paths
(see \Cref{app:case}), but extreme cases can still be challenging. Note that all LLM reasoning systems that first retrieve a Top-$K$ set of candidates will face the same challenge.

\vspace{-0.3em}
\section{Related Work}\label{sec:related_work}
\vspace{-0.3em}
\subsection{LLM-based Reasoning} 
Large language models (LLMs) are now widely adopted across research and industry—powering generation, retrieval, and decision-support systems at scale \citep{chang2024survey,minaee2024large,liu2025lune,liu2025recover}. LLM-based Reasoning, such as GPT~\citep{radford2019language,brown2020language}, LLaMA~\citep{touvron2023llama}, and PaLM~\citep{chowdhery2023palm}, have demonstrated impressive capabilities in diverse reasoning tasks, ranging from natural language inference to multi-hop question answering~\citep{yang2018hotpotqa}. A growing body of work focuses on enhancing the interpretability and reliability of LLM reasoning through \emph{symbolic path-based reasoning} over structured knowledge sources~\citep{sun2018open,cao2022kqa,hu2025cotel}. For example, Wei et al.~\citep{wei2022chain} proposed chain-of-thought prompting, which improves reasoning accuracy by encouraging explicit intermediate steps. Wang et al.~\citep{wang2022self} introduced self-consistency decoding, which aggregates multiple reasoning chains to improve robustness. 

Knowledge graphs are widely deployed in real-world systems (e.g., web search, recommendation, biomedicine) and constitute an active research area for representation learning and reasoning \citep{ji2021survey,zou2020survey,liu2023error,zhang2022contrastive}. In the context of knowledge graphs, recent efforts have explored hybrid neural-symbolic approaches to combine the structural expressiveness of graph reasoning with the generative power of LLMs. Fan et al.~\citep{luo2023reasoning} proposed Reasoning on Graphs (RoG), which first prompts LLMs to generate plausible symbolic relation paths and then retrieves and verifies these paths over knowledge graphs. Similarly, Khattab et al.~\citep{khattab2022demonstrate} leveraged demonstration-based prompting to guide LLM reasoning grounded in external knowledge. Despite their interpretability benefits, these methods rely heavily on neural encoders for path matching, incurring substantial computational and memory overhead, which limits scalability to large KGs or real-time applications.
\vspace{-0.3em}
\subsection{Hyperdimensional Computing (HDC)} 
\vspace{-0.3em}
HDC is an emerging computational paradigm inspired by the properties of high-dimensional representations in cognitive neuroscience~\citep{kanerva2009hyperdimensional,kanerva1997fully}. In HDC, information is represented as fixed-length high-dimensional vectors (hypervectors), and symbolic structures are manipulated through simple algebraic operations such as binding, bundling, and permutation~\citep{gayler2004vector}. These operations are inherently parallelizable and robust to noise, making HDC appealing for energy-efficient and low-latency computation.

HDC has been successfully applied in domains such as classification~\citep{rahimi2017hyperdimensional}, biosignal processing~\citep{moin2021wearable}, natural language understanding~\citep{maddali2023fusion}, and graph analytics~\citep{imani2019framework}. For instance, Imani et al.~\citep{imani2019framework} demonstrated that HDC can encode and process graph-structured data efficiently, enabling scalable similarity search and inference. Recent studies have also explored \emph{neuro-symbolic} integrations, where HDC complements neural networks to achieve interpretable yet computationally efficient models~\citep{imani2019adapthd,rahimi2017hyperdimensional}. However, the potential of HDC in large-scale reasoning over knowledge graphs, particularly when combined with LLMs, remains underexplored. Our work bridges this gap by leveraging HDC as a drop-in replacement for neural path matchers in LLM-based reasoning frameworks, thereby achieving both scalability and interpretability.

Existing KG-LLM reasoning frameworks typically rely on learned neural encoders or multi-call agent pipelines to score candidate paths or subgraphs, often with Transformers, GNNs, or repeated LLM calls. In contrast, our work keeps the retrieval module entirely encoder-free and training-free: {\ourmethod} replaces neural path scorers with HDC-based hypervector encodings and similarity, while remaining compatible with standard KG-LLM agents. Our goal is thus not to introduce new VSA theory, but to show that such carefully designed HDC representations can replace learned neural scorers in KG-LLM systems while preserving accuracy and substantially improving latency, memory footprint, and interpretability.

\section{Conclusion}
In this work, we presented \textbf{\ourmethod}, an encoder-free and
interpretable retrieval mechanism for path-based reasoning over knowledge graphs
with LLMs.
{\ourmethod} replaces neural path scorers with hyperdimensional Computing (HDC):
relation paths are encoded into order-aware GHRR hypervectors, ranked via
simple vector operations, and passed to a \emph{single} LLM adjudication step
that outputs answers together with cited supporting paths.
This Plan $\rightarrow$ Encode $\rightarrow$ Retrieve $\rightarrow$ Reason
design removes the need for costly neural encoders in the retrieval module and
shifts most computation into cheap, parallel hypervector operations. Experimental results on three standard Freebase-based KGQA benchmarks
(WebQSP, CWQ, GrailQA) show that \ourmethod attains competitive Hits@1 and F1
while substantially reducing end-to-end latency and GPU memory usage, and
produces faithful, path-grounded rationales that aid error analysis and
controllability.
Taken together, these findings indicate that carefully designed HDC representations are a practical substrate for efficient KG-LLM reasoning, offering a favorable accuracy-efficiency-interpretability trade-off. An important direction for future work is to apply {\method} to domain-specific graphs, such as UMLS and other biomedical or enterprise KGs, as well as to tasks beyond QA (e.g., fact checking or rule induction), and to explore how the HDC representations and retrieval pipeline can be adapted in these settings while retaining the same efficiency benefits.

\section*{Acknowledgements}
This work was supported in part by the DARPA Young Faculty Award, the National Science Foundation (NSF) under Grants \#2127780, \#2319198, \#2321840, \#2312517, and \#2235472, \#2431561, the Semiconductor Research Corporation (SRC), the Office of Naval Research through the Young Investigator Program Award, and Grants \#N00014-21-1-2225 and \#N00014-22-1-2067, Army Research Office Grant \#W911NF2410360. Additionally, support was provided by the Air Force Office of Scientific Research under Award \#FA9550-22-1-0253, along with generous gifts from Xilinx and Cisco.

\bibliography{ref}

\begin{thebibliography}{10}

\bibitem{achiam2023gpt}
Josh Achiam, Steven Adler, Sandhini Agarwal, Lama Ahmad, Ilge Akkaya, Florencia~Leoni Aleman, Diogo Almeida, Janko Altenschmidt, Sam Altman, Shyamal Anadkat, et~al.
\newblock Gpt-4 technical report.
\newblock {\em arXiv preprint arXiv:2303.08774}, 2023.

\bibitem{bollacker2008freebase}
Kurt Bollacker, Colin Evans, Praveen Paritosh, Tim Sturge, and Jamie Taylor.
\newblock Freebase: a collaboratively created graph database for structuring human knowledge.
\newblock In {\em Proceedings of the 2008 ACM SIGMOD international conference on Management of data}, pages 1247--1250, 2008.

\bibitem{brown2020language}
Tom~B Brown, Benjamin Mann, Nick Ryder, Melanie Subbiah, Jared Kaplan, Prafulla Dhariwal, Arvind Neelakantan, Pranav Shyam, Girish Sastry, Amanda Askell, et~al.
\newblock Language models are few-shot learners.
\newblock In {\em Advances in Neural Information Processing Systems (NeurIPS)}, volume~33, pages 1877--1901, 2020.

\bibitem{cao2022kqa}
Shulin Cao, Jiaxin Shi, Liangming Pan, Lunyiu Nie, Yutong Xiang, Lei Hou, Juanzi Li, Bin He, and Hanwang Zhang.
\newblock Kqa pro: A dataset with explicit compositional programs for complex question answering over knowledge base.
\newblock In {\em Proceedings of the 60th annual meeting of the Association for Computational Linguistics (volume 1: long papers)}, pages 6101--6119, 2022.

\bibitem{chang2024survey}
Yupeng Chang, Xu~Wang, Jindong Wang, Yuan Wu, Linyi Yang, Kaijie Zhu, Hao Chen, Xiaoyuan Yi, Cunxiang Wang, Yidong Wang, et~al.
\newblock A survey on evaluation of large language models.
\newblock {\em ACM transactions on intelligent systems and technology}, 15(3):1--45, 2024.

\bibitem{chowdhery2023palm}
Aakanksha Chowdhery, Sharan Narang, Jacob Devlin, Maarten Bosma, Gaurav Mishra, Adam Roberts, Paul Barham, Hyung~Won Chung, Charles Sutton, Sebastian Gehrmann, et~al.
\newblock Palm: Scaling language modeling with pathways.
\newblock {\em Journal of Machine Learning Research}, 24(240):1--113, 2023.

\bibitem{frady2021variable}
E.~Paxon Frady, Denis Kleyko, and Friedrich~T. Sommer.
\newblock Variable binding for sparse distributed representations: Theory and applications.
\newblock {\em Neural Computation}, 33(9):2207--2248, 2021.

\bibitem{gayler2004vector}
Ross~W Gayler.
\newblock Vector symbolic architectures answer jackendoff's challenges for cognitive neuroscience.
\newblock {\em arXiv preprint cs/0412059}, 2004.

\bibitem{gu2021beyond}
Yu~Gu, Sue Kase, Michelle Vanni, Brian Sadler, Percy Liang, Xifeng Yan, and Yu~Su.
\newblock Beyond iid: three levels of generalization for question answering on knowledge bases.
\newblock In {\em Proceedings of the web conference 2021}, pages 3477--3488, 2021.

\bibitem{he2021improving}
Gaole He, Yunshi Lan, Jing Jiang, Wayne~Xin Zhao, and Ji-Rong Wen.
\newblock Improving multi-hop knowledge base question answering by learning intermediate supervision signals.
\newblock In {\em Proceedings of the 14th ACM international conference on web search and data mining}, pages 553--561, 2021.

\bibitem{hu2025cotel}
Haotian Hu, Alex~Jie Yang, Sanhong Deng, Dongbo Wang, and Min Song.
\newblock Cotel-d3x: a chain-of-thought enhanced large language model for drug--drug interaction triplet extraction.
\newblock {\em Expert Systems with Applications}, 273:126953, 2025.

\bibitem{imani2019framework}
Mohsen Imani, Yeseong Kim, Sadegh Riazi, John Messerly, Patric Liu, Farinaz Koushanfar, and Tajana Rosing.
\newblock A framework for collaborative learning in secure high-dimensional space.
\newblock In {\em 2019 IEEE 12th International Conference on Cloud Computing (CLOUD)}, pages 435--446. IEEE, 2019.

\bibitem{imani2019adapthd}
Mohsen Imani, Justin Morris, Samuel Bosch, Helen Shu, Giovanni De~Micheli, and Tajana Rosing.
\newblock Adapthd: Adaptive efficient training for brain-inspired hyperdimensional computing.
\newblock In {\em 2019 IEEE Biomedical Circuits and Systems Conference (BioCAS)}, pages 1--4. IEEE, 2019.

\bibitem{ji2021survey}
Shaoxiong Ji, Shirui Pan, Erik Cambria, Pekka Marttinen, and Philip~S Yu.
\newblock A survey on knowledge graphs: Representation, acquisition, and applications.
\newblock {\em IEEE transactions on neural networks and learning systems}, 33(2):494--514, 2021.

\bibitem{jiang2023structgpt}
Jinhao Jiang, Kun Zhou, Zican Dong, Keming Ye, Wayne~Xin Zhao, and Ji-Rong Wen.
\newblock Structgpt: A general framework for large language model to reason over structured data.
\newblock {\em arXiv preprint arXiv:2305.09645}, 2023.

\bibitem{jiang2024kg}
Jinhao Jiang, Kun Zhou, Wayne~Xin Zhao, Yang Song, Chen Zhu, Hengshu Zhu, and Ji-Rong Wen.
\newblock Kg-agent: An efficient autonomous agent framework for complex reasoning over knowledge graph.
\newblock {\em arXiv preprint arXiv:2402.11163}, 2024.

\bibitem{jiang2022unikgqa}
Jinhao Jiang, Kun Zhou, Wayne~Xin Zhao, and Ji-Rong Wen.
\newblock Unikgqa: Unified retrieval and reasoning for solving multi-hop question answering over knowledge graph.
\newblock {\em arXiv preprint arXiv:2212.00959}, 2022.

\bibitem{kanerva2009hyperdimensional}
Pentti Kanerva.
\newblock Hyperdimensional computing: An introduction to computing in distributed representation with high-dimensional random vectors.
\newblock {\em Cognitive Computation}, 1(2):139--159, 2009.

\bibitem{kanerva1997fully}
Pentti Kanerva et~al.
\newblock Fully distributed representation.
\newblock {\em PAT}, 1(5):10000, 1997.

\bibitem{khattab2022demonstrate}
Omar Khattab, Keshav Santhanam, Xiang~Lisa Li, David Hall, Percy Liang, Christopher Potts, and Matei Zaharia.
\newblock Demonstrate-search-predict: Composing retrieval and language models for knowledge-intensive nlp.
\newblock {\em arXiv preprint arXiv:2212.14024}, 2022.

\bibitem{lewis2020rag}
Patrick Lewis, Ethan Perez, Aleksandra Piktus, Fabio Petroni, Vladimir Karpukhin, Naman Goyal, Heinrich Küttler, Mike Lewis, Wen-tau Yih, Tim Rockt{\"a}schel, et~al.
\newblock Retrieval-augmented generation for knowledge-intensive nlp.
\newblock In {\em NeurIPS}, 2020.

\bibitem{liu2025lune}
Yezi Liu, Hanning Chen, Wenjun Huang, Yang Ni, and Mohsen Imani.
\newblock Lune: Efficient llm unlearning via lora fine-tuning with negative examples.
\newblock In {\em Socially Responsible and Trustworthy Foundation Models at NeurIPS 2025}.

\bibitem{liu2025recover}
Yezi Liu, Hanning Chen, Wenjun Huang, Yang Ni, and Mohsen Imani.
\newblock Recover-to-forget: Gradient reconstruction from lora for efficient llm unlearning.
\newblock In {\em Socially Responsible and Trustworthy Foundation Models at NeurIPS 2025}.

\bibitem{liu2023error}
Yezi Liu, Qinggang Zhang, Mengnan Du, Xiao Huang, and Xia Hu.
\newblock Error detection on knowledge graphs with triple embedding.
\newblock In {\em 2023 31st European Signal Processing Conference (EUSIPCO)}, pages 1604--1608. IEEE, 2023.

\bibitem{luo2023reasoning}
Linhao Luo, Yuan-Fang Li, Gholamreza Haffari, and Shirui Pan.
\newblock Reasoning on graphs: Faithful and interpretable large language model reasoning.
\newblock {\em arXiv preprint arXiv:2310.01061}, 2023.

\bibitem{maddali2023fusion}
Raghavender Maddali.
\newblock Fusion of quantum-inspired ai and hyperdimensional computing for data engineering.
\newblock {\em Zenodo, doi}, 10, 2023.

\bibitem{miller2016key}
Alexander~H. Miller, Adam Fisch, Jesse Dodge, Amir{-}Hossein Karimi, Antoine Bordes, and Jason Weston.
\newblock Key-value memory networks for directly reading documents.
\newblock {\em CoRR}, abs/1606.03126, 2016.

\bibitem{minaee2024large}
Shervin Minaee, Tomas Mikolov, Narjes Nikzad, Meysam Chenaghlu, Richard Socher, Xavier Amatriain, and Jianfeng Gao.
\newblock Large language models: A survey.
\newblock {\em arXiv preprint arXiv:2402.06196}, 2024.

\bibitem{moin2021wearable}
Ali Moin, Alex Zhou, Abbas Rahimi, Ankita Menon, Simone Benatti, George Alexandrov, Samuel Tamakloe, Joash Ting, Naoya Yamamoto, Yasser Khan, et~al.
\newblock A wearable biosensing system with in-sensor adaptive machine learning for hand gesture recognition.
\newblock {\em Nature Electronics}, 4(1):54--63, 2021.

\bibitem{ouyang2022training}
Long Ouyang, Jeffrey Wu, Xu~Jiang, Diogo Almeida, Carroll Wainwright, Pamela Mishkin, Chong Zhang, Sandhini Agarwal, Katarina Slama, Alex Ray, et~al.
\newblock Training language models to follow instructions with human feedback.
\newblock {\em Advances in neural information processing systems}, 35:27730--27744, 2022.

\bibitem{plate1995holographic}
Tony~A Plate.
\newblock Holographic reduced representations.
\newblock {\em IEEE Transactions on Neural Networks}, 6(3):623--641, 1995.

\bibitem{press2022selfask}
Ofir Press, Muru Zhang, Sewon Min, Ludwig Schmidt, Noah~A. Smith, and Omer Levy.
\newblock Measuring and narrowing the compositionality gap in language models.
\newblock {\em arXiv:2210.03350}, 2022.
\newblock Self-Ask.

\bibitem{radford2019language}
Alec Radford, Jeffrey Wu, Rewon Child, David Luan, Dario Amodei, and Ilya Sutskever.
\newblock Language models are unsupervised multitask learners.
\newblock {\em OpenAI Blog}, 1(8):9, 2019.

\bibitem{rahimi2017hyperdimensional}
Abbas Rahimi, Pentti Kanerva, Jos{\'e} del~R Mill{\'a}n, and Jan~M Rabaey.
\newblock Hyperdimensional computing for noninvasive brain--computer interfaces: Blind and one-shot classification of eeg error-related potentials.
\newblock In {\em 10th EAI International Conference on Bio-Inspired Information and Communications Technologies}, page~19. European Alliance for Innovation (EAI), 2017.

\bibitem{saxena2020improving}
Apoorv Saxena, Aditay Tripathi, and Partha Talukdar.
\newblock Improving multi-hop question answering over knowledge graphs using knowledge base embeddings.
\newblock In {\em Proceedings of the 58th annual meeting of the association for computational linguistics}, pages 4498--4507, 2020.

\bibitem{shi2021transfernet}
Jiaxin Shi, Shulin Cao, Lei Hou, Juanzi Li, and Hanwang Zhang.
\newblock Transfernet: An effective and transparent framework for multi-hop question answering over relation graph.
\newblock {\em arXiv preprint arXiv:2104.07302}, 2021.

\bibitem{sui2024fidelis}
Yuan Sui, Yufei He, Nian Liu, Xiaoxin He, Kun Wang, and Bryan Hooi.
\newblock Fidelis: Faithful reasoning in large language model for knowledge graph question answering.
\newblock {\em arXiv preprint arXiv:2405.13873}, 2024.

\bibitem{sun2018open}
Haitian Sun, Tania Bedrax-Weiss, and William~W Cohen.
\newblock Open domain question answering using early fusion of knowledge bases and text.
\newblock In {\em Proceedings of the 2018 Conference on Empirical Methods in Natural Language Processing (EMNLP)}, pages 4231--4242, 2018.

\bibitem{sun2023think}
Jiashuo Sun, Chengjin Xu, Lumingyuan Tang, Saizhuo Wang, Chen Lin, Yeyun Gong, Lionel~M Ni, Heung-Yeung Shum, and Jian Guo.
\newblock Think-on-graph: Deep and responsible reasoning of large language model on knowledge graph.
\newblock {\em arXiv preprint arXiv:2307.07697}, 2023.

\bibitem{talmor2018web}
Alon Talmor and Jonathan Berant.
\newblock The web as a knowledge-base for answering complex questions.
\newblock {\em arXiv preprint arXiv:1803.06643}, 2018.

\bibitem{touvron2023llama}
Hugo Touvron, Thibaut Lavril, Gautier Izacard, Xavier Martinet, Marie-Anne Lachaux, Timothée Lacroix, Baptiste Rozière, Naman Goyal, Eric Hambro, Faisal Azhar, et~al.
\newblock Llama: Open and efficient foundation language models.
\newblock {\em arXiv preprint arXiv:2302.13971}, 2023.

\bibitem{wang2022self}
Xuezhi Wang, Jason Wei, Dale Schuurmans, Quoc Le, Ed~Chi, Sharan Narang, Aakanksha Chowdhery, and Denny Zhou.
\newblock Self-consistency improves chain of thought reasoning in language models.
\newblock In {\em Advances in Neural Information Processing Systems (NeurIPS)}, 2022.

\bibitem{wei2022chain}
Jason Wei, Xuezhi Wang, Dale Schuurmans, Maarten Bosma, Brian Ichter, Fei Xia, Ed~H Chi, Quoc~V Le, and Denny Zhou.
\newblock Chain-of-thought prompting elicits reasoning in large language models.
\newblock In {\em Advances in Neural Information Processing Systems (NeurIPS)}, 2022.

\bibitem{xu2024generate}
Yao Xu, Shizhu He, Jiabei Chen, Zihao Wang, Yangqiu Song, Hanghang Tong, Guang Liu, Kang Liu, and Jun Zhao.
\newblock Generate-on-graph: Treat llm as both agent and kg in incomplete knowledge graph question answering.
\newblock {\em arXiv preprint arXiv:2404.14741}, 2024.

\bibitem{yang2018hotpotqa}
Zhilin Yang, Peng Qi, Saizheng Zhang, Yoshua Bengio, William~W Cohen, Ruslan Salakhutdinov, and Christopher~D Manning.
\newblock Hotpotqa: A dataset for diverse, explainable multi-hop question answering.
\newblock {\em arXiv preprint arXiv:1809.09600}, 2018.

\bibitem{yao2023react}
Shunyu Yao, Dian Yang, Run-Ze Cui, and Karthik Narasimhan.
\newblock React: Synergizing reasoning and acting in language models.
\newblock In {\em ICLR}, 2023.

\bibitem{yao2024tree}
Shunyu Yao, Dian Zhao, Luyu Yu, and Karthik Narasimhan.
\newblock Tree of thoughts: Deliberate problem solving with large language models.
\newblock {\em arXiv:2305.10601}, 2024.

\bibitem{yeung2024generalized}
Calvin Yeung, Zhuowen Zou, and Mohsen Imani.
\newblock Generalized holographic reduced representations.
\newblock {\em arXiv preprint arXiv:2405.09689}, 2024.

\bibitem{yih2016value}
Wen-tau Yih, Matthew Richardson, Christopher Meek, Ming-Wei Chang, and Jina Suh.
\newblock The value of semantic parse labeling for knowledge base question answering.
\newblock In {\em Proceedings of the 54th Annual Meeting of the Association for Computational Linguistics (Volume 2: Short Papers)}, pages 201--206, 2016.

\bibitem{zhang2022subgraph}
Jing Zhang, Xiaokang Zhang, Jifan Yu, Jian Tang, Jie Tang, Cuiping Li, and Hong Chen.
\newblock Subgraph retrieval enhanced model for multi-hop knowledge base question answering.
\newblock {\em arXiv preprint arXiv:2202.13296}, 2022.

\bibitem{zhang2022contrastive}
Qinggang Zhang, Junnan Dong, Keyu Duan, Xiao Huang, Yezi Liu, and Linchuan Xu.
\newblock Contrastive knowledge graph error detection.
\newblock In {\em Proceedings of the 31st ACM International Conference on Information \& Knowledge Management}, pages 2590--2599, 2022.

\bibitem{zou2020survey}
Xiaohan Zou.
\newblock A survey on application of knowledge graph.
\newblock In {\em Journal of Physics: Conference Series}, volume 1487, page 012016. IOP Publishing, 2020.

\end{thebibliography}
\bibliographystyle{plain}

\clearpage
\appendix
\section{Notation}\label{app:not}
\begin{table*}[h]
\centering
\small
\setlength{\tabcolsep}{6pt}
\renewcommand{\arraystretch}{1.18}
\begin{tabularx}{\textwidth}{lX}
\Xhline{1.1pt}
\textbf{Notation} & \textbf{Definition} \\
\hline
\hline
$\mathcal{G}=(\mathcal{V},\mathcal{E})$ & Knowledge graph with entity set $\mathcal{V}$ and edge set $\mathcal{E}$. \\
$\mathcal{Z}$ & Set of relation schemas/path templates. \\
$q$, $a$ & Input question and (predicted) answer. \\
$e$, $r$ & An entity and a relation (schema edge), respectively. \\
$z=(r_1,\ldots,r_\ell)$ & A relation path; $|z|=\ell$ denotes path length. \\
$\mathcal{Z}_{\text{cand}}$ & Candidate path set instantiated from $\mathcal{G}$.\\ 
$N = |Z_{\text{cand}}|$ &  The number of candidate paths instantiated from the KG for a given query.\\
$L_{\max}$, $B$, $K$ & Max plan depth, BFS beam width, and number of retrieved paths kept after pruning. \\
$d$, $D$, $m$ & Hypervector dimension, \# of GHRR blocks, and block size (unitary $m{\times}m$); flattened $d=Dm^2$. \\
$\mathbf{v}_x$ & Hypervector for symbol $x$ (entity/relation/path). \\
$\mathbf{v}_q$, $\mathbf{v}_z$ & Query-plan hypervector and a candidate-path hypervector. \\
$\mathbf{H}=[A_1;\dots;A_D]$ & A GHRR hypervector with unitary blocks $A_j\in\mathrm{U}(m)$. \\
$A^\ast$ & Conjugate transpose (unitary inverse) of a block $A$. \\
$\bbind$ & GHRR \emph{blockwise binding} operator (matrix product per block). \\
$\langle A,B\rangle_F$ & Frobenius inner product $\mathrm{tr}(A^\ast B)$; $\|A\|_F$ is the Frobenius norm. \\
$\mathrm{sim}(\cdot,\cdot)$ & Blockwise cosine similarity used for HD retrieval. \\
$s(z)$ & Calibrated retrieval score; $\alpha,\beta,\lambda$ are calibration hyperparameters; $\mathrm{IDF}(z)$ is an inverse-frequency weight. \\
$\mathcal{M}$, $M$ & Distractor set and its size $M=\lvert\mathcal{M}\rvert$ (used in capacity bounds). \\
$\epsilon,\delta$ & Tolerance and failure probability in the concentration/union bounds. \\
$c$ & Absolute constant in the sub-Gaussian tail bound. \\
\Xhline{1.1pt}
\end{tabularx}
\caption{Notation used throughout the paper.}
\label{tab:notation}
\end{table*}

\clearpage
\section{Algorithm}

\begin{algorithm}[H]
\caption{\textsc{HD-Retrieve}: Hyperdimensional Top-$K$ Path Retrieval}
\label{alg:hd-retrieve}
\KwIn{question $q$; KG $\mathcal{G}$; relation schemas $\mathcal{Z}$; max depth $L_{\max}$; beam width $B$; calibration $(\alpha,\beta,\lambda)$; Top-$K$}
\KwOut{Top-$K$ reasoning paths $\mathcal{P}_K$ and their scores}

\BlankLine
\textbf{Plan (schema-level):}\\
Construct a relation-schema graph over $\mathcal{Z}$ and run constrained BFS up to depth $L_{\max}$ with beam width $B$ to obtain a small set of type-consistent relation plans $\mathcal{Z}_q \subseteq \mathcal{Z}$ for $q$. \\

\textbf{Encode Query:}\\
Pick a plan $z_q \in \mathcal{Z}_q$ and encode it by GHRR binding
\[
\mathbf{v}_q \leftarrow \mathrm{BindPath}(z_q)
= \bigbbind_{r \in z_q} \mathbf{v}_r,
\]
followed by blockwise normalization. \tcp*{plan-based query hypervector; no unbinding}

\BlankLine
\textbf{Instantiate Candidates (entity-level):}\\
Initialize $\mathcal{P}(q) \leftarrow \emptyset$. \\
\For{$z \in \mathcal{Z}_q$}{
  Instantiate concrete KG paths consistent with schema $z$ by matching its relation pattern to edges in $\mathcal{G}$ or by a constrained BFS on $\mathcal{G}$ (depth $\le L_{\max}$, beam width $B$);\\
  Add all instantiated paths to $\mathcal{P}(q)$.
}
Deduplicate paths in $\mathcal{P}(q)$ and enforce type consistency. \\

\BlankLine
\For{$p \in \mathcal{P}(q)$}{
  Let $z(p) = (r_1,\dots,r_\ell)$ be the relation sequence of path $p$. \\[0.2em]
  \textbf{Encode Candidate:}
  \[
  \mathbf{v}_p \leftarrow \mathrm{BindPath}(z(p))
  = \bigbbind_{r \in z(p)} \mathbf{v}_r
  \]
  with blockwise normalization. \\[0.2em]
  \textbf{Score (blockwise cosine):}
  \[
  s_{\mathrm{cos}}(p) \leftarrow \mathrm{sim}(\mathbf{v}_q,\mathbf{v}_p)
  \quad\text{(Eq.~\eqref{eq:block-cos})}
  \]
  \textbf{Calibrate (optional):}
  \[
  s(p) \leftarrow s_{\mathrm{cos}}(p)
    + \alpha\,\mathrm{IDF}(\text{schema}(z(p)))
    - \beta\,\lambda^{|z(p)|}
  \]
}
\Return Top-$K$ paths in $\mathcal{P}(q)$ ranked by $s(p)$ as $\mathcal{P}_K$.\\[0.3em]
\tcp{All steps above are symbolic; no additional LLM calls beyond the final reasoning step.}
\end{algorithm}

\clearpage
\section{Prompt Template for One-shot Reasoning}
\label{app:prompt}

\begin{table}[h]
\centering
\small
\setlength{\tabcolsep}{6pt}
\renewcommand{\arraystretch}{1.2}

\begin{tabularx}{\linewidth}{@{}lX@{}}
\toprule
\textbf{System} &
\textit{You are a careful reasoner. Only use the provided KG reasoning paths as evidence.
Cite the most relevant path(s) and answer concisely.} \\
\midrule

\textbf{User} &
\textbf{Question:} ``\${QUESTION}''

\medskip

\textbf{Retrieved paths (Top-$K$):}
\begin{enumerate}[leftmargin=1.2em,itemsep=0pt,topsep=2pt,parsep=0pt]
\item \${PATH\_1} %
\item \${PATH\_2}
\item \dots
\item \${PATH\_K}
\end{enumerate} \\
\midrule

\makebox[\widthof{\textbf{Assistant (required format)}}][l]{\textbf{Assistant (required format)}} &
\textbf{Answer:} \${SHORT\_ANSWER}

\medskip

\textbf{Supporting path(s):} [indexes from the list above]

\medskip

\textbf{Rationale (1--2 sentences):} why those paths imply the answer. \\
\bottomrule
\end{tabularx}

\caption{Prompt template for KG path–grounded QA.}
\label{tab:prompt-template}
\end{table}
\section{Additional Theoretical Support}
\label{app:theory-support}

\subsection{Near-Orthogonality and Capacity in a Toy Rademacher Model}

To build intuition for why high-dimensional hypervectors are effective for
retrieval, we first consider a simplified, classical VSA setting: real
Rademacher hypervectors with element-wise binding.
Our actual method uses complex GHRR hypervectors with blockwise binding
(Section~\ref{sec:hv-init}, Proposition~\ref{prop:near-orth}), but the same
sub-Gaussian concentration arguments apply.

We justify the use of high-dimensional hypervectors in \textbf{\ourmethod} by
showing that (i) random hypervectors are nearly orthogonal with high
probability, and (ii) this property is preserved under binding, yielding
exponential concentration that enables accurate retrieval at scale.

\paragraph{Setup.}
Let each entity/relation be encoded as a Rademacher hypervector
$\mathbf{x}\in\{-1,+1\}^d$ with i.i.d.\ entries.
For two independent hypervectors $\mathbf{x},\mathbf{y}$, define cosine
similarity
$\cos(\mathbf{x},\mathbf{y}) = \frac{\langle \mathbf{x},\mathbf{y}\rangle}
{\|\mathbf{x}\|\,\|\mathbf{y}\|}$.
Since $\|\mathbf{x}\|=\|\mathbf{y}\|=\sqrt{d}$, we have
$\cos(\mathbf{x},\mathbf{y})=\frac{1}{d}\sum_{k=1}^{d} x_k y_k$.

\begin{proposition}[Near-orthogonality of random hypervectors]
\label{prop:nearorth}
For any $\epsilon\in(0,1)$,
\[
\Pr\!\left(\big|\cos(\mathbf{x},\mathbf{y})\big|>\epsilon\right)
\;\le\; 2\exp\!\left(-\tfrac{1}{2}\epsilon^2 d\right).
\]
\end{proposition}
\begin{proof}
Each product $Z_k=x_k y_k$ is i.i.d.\ Rademacher with $\mathbb{E}[Z_k]=0$
and $|Z_k|\le 1$.
By Hoeffding’s inequality,
$\Pr\!\left(\left|\sum_{k=1}^d Z_k\right|>\epsilon d\right)\le
2\exp(-\epsilon^2 d/2)$.
Divide both sides by $d$ to obtain the claim.
\end{proof}

\begin{lemma}[Closure under binding]
\label{lem:binding}
Let $\mathbf{r}_1,\dots,\mathbf{r}_n$ be independent Rademacher hypervectors
and define binding (element-wise product)
$\mathbf{p}=\mathbf{r}_1\odot\cdots\odot\mathbf{r}_n$.
Then $\mathbf{p}$ is also a Rademacher hypervector.
Moreover, if $\mathbf{s}$ is independent of at least one $\mathbf{r}_i$ used in
$\mathbf{p}$, then $\mathbf{p}$ and $\mathbf{s}$ behave as independent
Rademacher hypervectors and
$\mathbb{E}[\cos(\mathbf{p},\mathbf{s})]=0$.
\end{lemma}
\begin{proof}
Each coordinate $p_k=\prod_{i=1}^{n} r_{i,k}$ is a product of independent
Rademacher variables, hence Rademacher.
If $\mathbf{s}$ is independent of some $r_{j}$, then $p_k s_k$ has zero mean
and remains bounded, so the same Hoeffding-type bound as in
Proposition~\ref{prop:nearorth} applies.
\end{proof}

\begin{theorem}[Separation and error bound for hypervector retrieval]
\label{thm:separation}
Let the query hypervector be
$\mathbf{q}=\mathbf{r}_1\odot\cdots\odot\mathbf{r}_n$ and consider a candidate
set containing the true path $\mathbf{p}^\star=\mathbf{q}$ and
$M$ distractors $\{\mathbf{p}_i\}_{i=1}^{M}$, where each distractor differs
from $\mathbf{q}$ in at least one relation (thus satisfies
Lemma~\ref{lem:binding}).
Then for any $\epsilon\in(0,1)$ and $\delta\in(0,1)$, if
\[
d \;\ge\; \frac{2}{\epsilon^2}\,\log\!\Big(\frac{2M}{\delta}\Big),
\]
we have, with probability at least $1-\delta$,
\[
\cos(\mathbf{q},\mathbf{p}^\star)=1
\quad\text{and}\quad
\max_{1\le i\le M}\big|\cos(\mathbf{q},\mathbf{p}_i)\big|\le \epsilon .
\]
\end{theorem}
\begin{proof}
By construction, $\mathbf{p}^\star=\mathbf{q}$, hence cosine $=1$.
For each distractor $\mathbf{p}_i$, Lemma~\ref{lem:binding} implies that
$\mathbf{q}$ and $\mathbf{p}_i$ behave as independent Rademacher
hypervectors; applying Proposition~\ref{prop:nearorth},
$\Pr(|\cos(\mathbf{q},\mathbf{p}_i)|>\epsilon)\le 2e^{-\epsilon^2 d/2}$.
A union bound over $M$ distractors yields
$\Pr(\max_i|\cos(\mathbf{q},\mathbf{p}_i)|>\epsilon)\le
2M e^{-\epsilon^2 d/2}\le \delta$ under the stated condition on $d$.
\end{proof}

\section{Additional Proofs and Tail Bounds}
\label{app:proofs}

\begin{proof}[Details for Prop.~\ref{prop:near-orth}]
Here, we sketch the argument for the GHRR near-orthogonality bound used in the
main text.
We view each GHRR block as a unitary matrix with i.i.d.\ phase (or signed)
entries, so blockwise products preserve the unit norm and keep coordinates sub-Gaussian.
Let $X=\frac{1}{d}\sum_{j=1}^d \xi_j$ with $\xi_j$ i.i.d., mean zero, and
$\psi_2$-norm bounded.
Applying Hoeffding/Bernstein,
$\Pr(|X|\ge\epsilon)\le 2\exp(-c d \epsilon^2)$ for some absolute constant
$c>0$, which yields the stated result after $\ell_2$ normalization.
Unitary blocks ensure no variance blow-up under binding depth; see
also~\cite{plate1995holographic,kanerva2009hyperdimensional} for stability of
holographic codes. \qedhere
\end{proof}

\section{Optional Prompt-Based Schema Refinement}\label{app:prompt-refine}
As described in \Cref{sec:path_construction}, all results in \Cref{sec:exp} use only schema-based enumeration of relation schemas, without any additional LLM calls beyond the final reasoning step.
For completeness, we describe here an optional extension that refines schema plans with a lightweight prompt.

Given a small set of candidate relation schemas $\{z^{(1)}, \dots, z^{(M)}\}$ obtained from enumeration, we first verbalize each schema into a short natural-language description (e.g., by mapping each relation type $r$ to a phrase and concatenating them). We then issue a single prompt of the form:

\begin{quote}
\small
Given the question $q$ and the following candidate relation patterns: (1) \texttt{[schema 1]}, (2) \texttt{[schema 2]}, \dots, which $K$ patterns are most relevant for answering $q$? Please output only the indices.
\end{quote}

The LLM outputs a small subset of indices, which we use to select the top-$K$ schemas $\{z^{(i_1)},\dots,z^{(i_K)}\}$. These refined schemas are then instantiated into concrete KG paths and encoded into hypervectors exactly as in the main method.

We emphasize that this refinement is \emph{not} used in any of our reported experiments, and that it can be implemented with at most one additional short LLM call per query. The main system studied in this paper, therefore, remains a single-call KG-LLM pipeline in all empirical results.

\section{Dataset Introduction}
\label{appendix:datasets}
We provide detailed descriptions of the three benchmark datasets used in our experiments:

WebQuestionsSP (WebQSP).
\noindent\textbf{WebQuestionsSP (WebQSP)}~\citep{yih2016value} consists of $4{,}737$ questions, where each question is manually annotated with a topic entity and a SPARQL query over Freebase. The answer entities are within a maximum of $2$ hops from the topic entity. Following prior work~\citep{sun2018open}, we use the standard train/validation/test splits released by GraftNet and the same Freebase subgraph for fair comparison.

\noindent\textbf{Complex WebQuestions–SP (CWQ-SP)}~\citep{talmor2018web} is the Freebase/SPARQL–annotated variant of CWQ, aligning each question to a topic entity and an executable SPARQL query over a cleaned Freebase subgraph. Questions are created by compositional expansions of WebQSP (adding constraints, joins, and longer paths), and typically require up to $4$-hop reasoning. We use the standard train/dev/test split released with CWQ-SP for fair comparison.

\noindent\textbf{GrailQA}~\citep{gu2021beyond} is a large-scale KGQA benchmark with $64{,}331$ questions. It focuses on evaluating generalization in multi-hop reasoning across three distinct settings: i.i.d., compositional, and zero-shot. Each question is annotated with a corresponding logical form and answer, and the underlying KG is a cleaned subset of Freebase. We follow the official split provided by the authors for fair comparison. \emph{In our experiments, we evaluate on the official \textbf{dev} set. The dev set is the authors’ held-out split from the same cleaned Freebase graph and mirrors the three generalization settings; it is commonly used for ablations and model selection when the test labels are held out.}

We follow the unified Freebase protocol~\citep{bollacker2008freebase}, which contains approximately $88$ million entities, $20$ thousand relations, and $126$ million triples. The official Hits@1/F1 scripts. 
For \textbf{GrailQA}, numbers in the main results are reported on the \textbf{dev} split (and additionally on its \emph{IID} subset); many recent works adopt dev evaluation due to test server restrictions. WebQSP has no official dev split under this setting. Additional statistics, including the number of reasoning hops and answer entities, are shown in~\Cref{tab:dataset_stats}.

\begin{table}[t]
\centering
\small
\begin{threeparttable}
\setlength{\tabcolsep}{6pt}
\renewcommand{\arraystretch}{1.15}
\begin{tabular}{l||c c c c c}
\Xhline{1.2pt}
\rowcolor{tablehead!20}\textbf{Dataset} & \textbf{Train} & \textbf{Dev} & \textbf{Test} & \textbf{Typical hops} & \textbf{KG} \\
\hline\hline
\rowcolor{gray!08}WebQSP~\citep{yih2016value} & 3{,}098 & -- & 1{,}639 & 1--2 & Freebase \\
CWQ~\citep{talmor2018web} & 27{,}734 & 3{,}480 & 3{,}475 & 2--4 & Freebase \\
\rowcolor{gray!08}GrailQA~\citep{gu2021beyond} & 44{,}337 & 6{,}763 & 13{,}231 & 1--4 & Freebase \\
\Xhline{1.2pt}
\end{tabular}
\end{threeparttable}
\caption{Statistics of Freebase-based KGQA datasets used in our experiments.}\label{tab:dataset_stats}
\end{table}

\section{Detailed Baseline Descriptions}
\label{appendix:baselines}
We categorize the baseline methods into four groups and describe each group below.
\vspace{-1em}
\subsection{Embedding-based methods}
\begin{fullitemize}
\item \textbf{KV-Mem}~\citep{miller2016key} uses a key-value memory architecture to store knowledge triples and performs multi-hop reasoning through iterative memory operations.  
\item \textbf{EmbedKGQA}~\citep{saxena2020improving} formulates KGQA as an entity-linking task and ranks entity embeddings using a question encoder.  
\textbf{NSM}~\citep{he2021improving} adopts a sequential program execution framework over KG relations, learning to construct and execute reasoning paths.  
\item \textbf{TransferNet}~\citep{shi2021transfernet} builds on GraftNet by incorporating both relational and text-based features, enabling interpretable step-wise reasoning over entity graphs.
\end{fullitemize}
\vspace{-1em}
\subsection{Retrieval-augmented methods}
\begin{fullitemize}
\item \textbf{GraftNet}~\citep{sun2018open} retrieves question-relevant subgraphs and applies GNNs for reasoning over linked entities.  
\item \textbf{SR+NSM}~\citep{zhang2022subgraph} retrieves relation-constrained subgraphs and runs NSM over them to generate answers.  
\item \textbf{SR+NSM+E2E}~\citep{zhang2022subgraph} further optimizes SR+NSM via end-to-end training of the retrieval and reasoning modules.  
\item \textbf{UniKGQA}~\citep{jiang2022unikgqa} unifies entity retrieval and graph reasoning into a single LLM-in-the-loop architecture, achieving strong performance with reduced pipeline complexity.
\end{fullitemize}
\vspace{-1em}
\subsection{Pure LLMs}
\begin{fullitemize}
\item \textbf{ChatGPT}~\citep{ouyang2022training}, \textbf{Davinci-003}~\citep{ouyang2022training}, and \textbf{GPT-4}~\citep{achiam2023gpt} serve as closed-book baselines using few-shot or zero-shot prompting.  
\item \textbf{StructGPT}~\citep{jiang2023structgpt} generates structured reasoning paths in natural language form, then executes them step by step.  
\item \textbf{ROG}~\citep{luo2023reasoning} reasons over graph-based paths with alignment to LLM beliefs.  
\item \textbf{Think-on-Graph}~\citep{sun2023think} prompts the LLM to search symbolic reasoning paths over a KG and use them for multi-step inference.
\end{fullitemize}
\vspace{-1em}
\subsection{LLMs + KG methods}
\begin{fullitemize}
\item \textbf{GoG}~\citep{xu2024generate} adopts a plan-then-retrieve paradigm, where an LLM generates reasoning plans and a KG subgraph is retrieved accordingly.  
\item \textbf{KG-Agent}~\citep{jiang2024kg} turns the KGQA task into an agent-style decision process using graph environment feedback.  
\item \textbf{FiDeLiS}~\citep{sui2024fidelis} fuses symbolic subgraph paths with LLM-generated evidence, filtering hallucinated reasoning chains.  
\item \textbf{{\ourmethod}} (ours) proposes a vector-symbolic integration pipeline where top-$K$ relation paths are selected by vector matching and adjudicated by an LLM, combining symbolic controllability with neural flexibility.
\end{fullitemize}

\section{Detailed Experimental Setups}
\label{app:setup}
We follow a unified evaluation protocol: all methods are evaluated on the
Freebase KG with the official \textit{Hits@1}/\textit{F1} scripts for WebQSP,
CWQ, and GrailQA (dev, IID), so that the numbers are directly comparable across
systems.
Whenever possible, we adopt the official results reported by prior work under
the same setting.
Concretely, we take numbers for KV-Mem \citep{miller2016key}, GraftNet
\citep{sun2018open}, EmbedKGQA \citep{saxena2020improving}, NSM
\citep{he2021improving}, TransferNet \citep{shi2021transfernet}, SR+NSM and its
end-to-end variant (SR+NSM+E2E) \citep{zhang2022subgraph}, UniKGQA
\citep{jiang2022unikgqa}, RoG \citep{luo2023reasoning}, StructGPT
\citep{jiang2023structgpt}, Think-on-Graph \citep{sun2023think}, GoG
\citep{xu2024generate}, and FiDeLiS \citep{sui2024fidelis} from their papers or
consolidated tables under equivalent Freebase protocols.
We further include a pure-LLM category (ChatGPT, Davinci-003, GPT-4) using the
unified results reported by KG-Agent \citep{jiang2024kg}; note that its GrailQA scores are on the dev split.
For KG-Agent itself, we use the official numbers reported in its paper~\citep{jiang2024kg}.

For {\ourmethod}, all LLMs and sentence encoders are used \emph{off the shelf} with
no fine-tuning.
The block size $m$ of the unitary blocks is chosen from a small range motivated by the VSA literature; following GHRR, we fix a small block size $m = 4$ and
primarily tune the overall dimensionality $d$.
Prior work on GHRR~\citep{yeung2024generalized} shows that, for a fixed total
dimension $d$, moderate changes in $m$ trade off non-commutativity and
saturation behaviour, but do not lead to instability.
In our experiments, we therefore treat $d$ as the main tuning parameter (set on
the dev split), while fixing $m$ across all runs.
Additional hyperparameters for candidate enumeration and calibration are
detailed in \Cref{app:ana_eff,app:calib}.

\subsection{Additional Analytic Efficiency}\label{app:ana_eff}
\begin{table}[h]
\centering
\small
\renewcommand{\arraystretch}{1.1}
\resizebox{\linewidth}{!}{
\begin{tabular}{l||cccc}
\Xhline{1.2pt}
\rowcolor{tablehead!20}
\textbf{Method} & \textbf{\# LLM calls / query} & \textbf{Planning depth} & \textbf{Retrieval fanout/beam} & \textbf{Executor/Tools} \\
\hline\hline
\rowcolor{gray!10}KV-Mem \citep{miller2016key}              & $0$ & multi-hop (learned) & moderate & Yes (neural mem) \\
EmbedKGQA \citep{saxena2020improving}     & $0$ & multi-hop (seq)     & moderate & No \\
\rowcolor{gray!10}NSM \citep{he2021improving}               & $0$ & multi-hop (neural)  & moderate & Yes (neural executor) \\
TransferNet \citep{shi2021transfernet}    & $0$ & multi-hop           & moderate & No \\
\rowcolor{gray!10}GraftNet \citep{sun2018open}              & $0$ & multi-hop           & graph fanout        & No \\
SR+NSM \citep{zhang2022subgraph}          & $0$ & multi-hop           & subgraph (beam)     & Yes (neural exec) \\
\rowcolor{gray!10}SR+NSM+E2E \citep{zhang2022subgraph}      & $0$ & multi-hop           & subgraph (beam)     & Yes (end-to-end) \\
ChatGPT \citep{ouyang2022training}        & $1$ & $0$                 & n/a                 & No \\
\rowcolor{gray!10}Davinci-003 \citep{ouyang2022training}    & $1$ & $0$                 & n/a                 & No \\
GPT-4 \citep{achiam2023gpt}              & $1$ & $0$                 & n/a                 & No \\
\rowcolor{gray!10}UniKGQA \citep{jiang2022unikgqa}          & $1\text{--}2$ & shallow & small/merged        & No (unified model) \\
StructGPT \citep{jiang2023structgpt}      & $1\text{--}2$ & $1$     & n/a                 & Yes (tool use) \\
\rowcolor{gray!10}RoG \citep{luo2023reasoning}              & $\approx d \!\times\! b$ & $d$ & $b$ (per step)    & No (LLM scoring) \\
Think-on-Graph \citep{sun2023think}       & $3\text{--}6$ & multi   & small/beam          & Yes (plan \& react) \\
\rowcolor{gray!10}GoG \citep{xu2024generate}                & $3\text{--}5$ & multi   & small/iterative     & Yes (generate-retrieve loop) \\
KG-Agent \citep{jiang2024kg}              & $3\text{--}8$ & multi   & small               & Yes (agent loop) \\
\rowcolor{gray!10}FiDeLiS \citep{sui2024fidelis}            & $1\text{--}3$ & shallow & small               & Optional (verifier) \\
\rowcolor[HTML]{D7F6FF}
\textbf{\ourmethod (ours)}                    & $\mathbf{1}$ (final only) & $\mathbf{0}$ & vector ops only & No (vector ops) \\
\Xhline{1.2pt}
\end{tabular}}
\caption{Full analytical comparison (no implementation). Ranges reflect algorithm design; $d$ and $b$ denote planning depth and beam/fanout as specified in RoG, which uses beam-search with $B=3$ and path length bounded by dataset hops (WebQSP$\leq2$, CWQ$\leq4$).}
\label{tab:analytic_full}
\end{table}

\subsection{More Hyperparameter Tuning Details}\label{app:calib}
We sweep $\alpha,\beta \in \{0,0.1,0.2,\dots,0.5\}$ and $\lambda \in \{0.6,0.7,0.8,0.9\}$ and pick the best-performing triple on the validation Hits@$1$ for each dataset.
\begin{table}[h]
\centering
\begin{tabular}{l||ccc}
\Xhline{1.2pt}
\rowcolor{tablehead!20}\textbf{Dataset} & $\mathbf{\alpha}$ & $\mathbf{\beta}$ & \textbf{$\lambda$} \\
\hline\hline
\rowcolor{gray!08}WebQSP & 0.2 & 0.1 & 0.8 \\
CWQ    & 0.3 & 0.1 & 0.8 \\
\rowcolor{gray!08}GrailQA & 0.2 & 0.2 & 0.8 \\
\Xhline{1.2pt}
\end{tabular}
\caption{Calibration hyperparameters $(\alpha,\beta,\lambda)$ used for each dataset.}
\label{tab:calib}
\end{table}

\clearpage
\section{Additional Experiments}
\subsection{Scoring Metric}
\begin{figure*}[h]
\vspace{-0.6em}
\centering
\includegraphics[width=\linewidth]{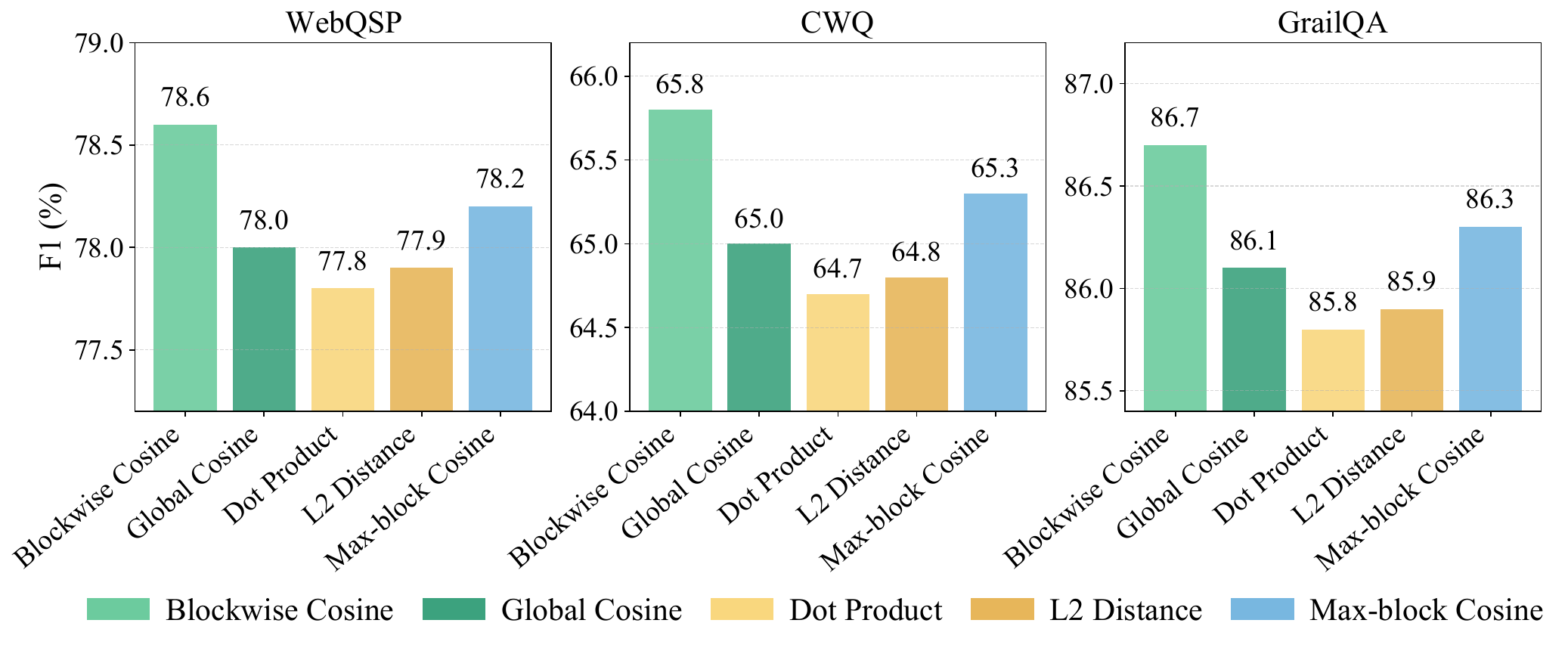}
\vspace{-2em}
\caption{\textbf{Scoring measurement ablation.} We evaluate F1 (\%) on WebQSP, CWQ, and GrailQA using different scoring strategies in our model. \ourmethod achieves the best or competitive results when using blockwise cosine similarity, highlighting its effectiveness in capturing fine-grained matching signals across vector blocks.}
\label{fig:attack}
\end{figure*}

\subsection{Text-projection Variant}\label{app:ablation_text_pro}
We use SBERT as the sentence encoder, so $d_t$ is fixed to the encoder hidden size $768$, and $P$ is sampled once from $\mathcal{N}(0, 1/d_t)$ and kept fixed for all experiments.
\begin{table}[h]
\vspace{-0.5em}
\centering
\small
\renewcommand{\arraystretch}{1.0}
\begin{tabular}{l||cc}
\Xhline{1.2pt}
\rowcolor{tablehead!20}\textbf{Final step} & \textbf{WebQSP} & \textbf{CWQ} \\
\hline\hline
\rowcolor{gray!10}{\ourmethod} (text-projection query)             & 83.4 & 69.8 \\
{\ourmethod} (plan-based, default)  & \textbf{86.2} & \textbf{71.5} \\
\Xhline{1.2pt}
\end{tabular}
\caption{Comparison of query encoding variants in {\ourmethod}. We report Hits@$1$ on WebQSP and CWQ for the default plan-based encoding and the text-projection variant.}
\label{tab:text_pro_var}
\end{table}

\clearpage

\subsection{Additional Visualization}
\begin{figure*}[h]
\vspace{-0.6em}
\centering
\includegraphics[width=\linewidth]{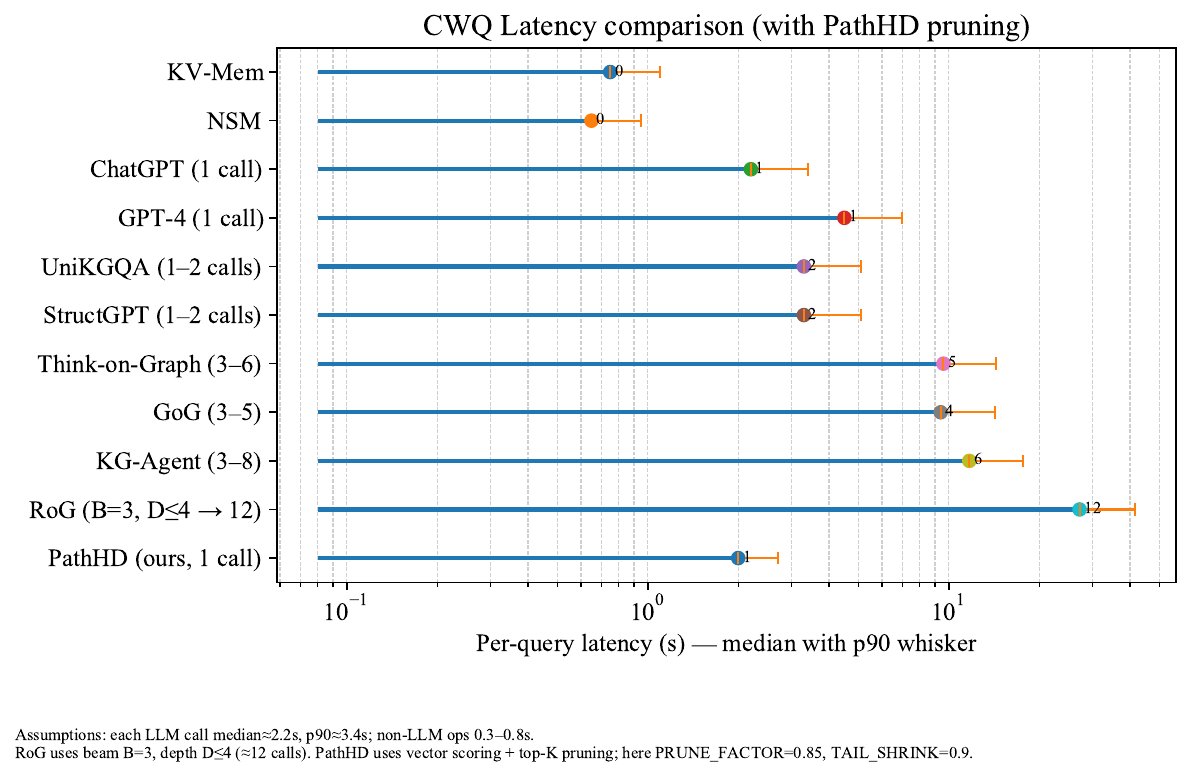}
\vspace{-1.9em}
\caption{CWQ latency comparison (lollipop). Dots indicate median per-query latency; right whiskers show the 90th percentile (p90). The x-axis is log-scaled. Values are \emph{estimated} under a unified setup: per-LLM-call median $\approx$ 2.2\,s and p90 $\approx$ 3.4\,s; non-LLM operations add 0.3-0.8\,s. RoG follows beam width $B{=}3$ with depth bounded by dataset hops ($D{\le}4$, $\approx$12 calls), whereas \textbf{\ourmethod} uses a single LLM call plus vector operations for scoring.}\label{fig:lollipop}
\vspace{-1.7em}
\end{figure*}
\clearpage
\subsection{Effect of Backbone Models}
Performance across different LLM backbones is shown in \Cref{tab:backbone}.
\begin{table}[h]
\centering
\small
\setlength{\tabcolsep}{2pt}
\renewcommand{\arraystretch}{1.15}
\caption{Performance across different LLM backbones. Each block fixes the backbone and varies the reasoning framework: a pure LLM control (CoT), our single-call \ourmethod, and 1-2 multi-step LLM+KG baselines. Metrics follow the unified Freebase setup.}
\label{tab:backbone}
\resizebox{\textwidth}{!}{
\begin{tabular}{l l cc c c}
\Xhline{1.2pt}
\rowcolor{tablehead!20}
\textbf{Backbone} & \textbf{Method} & \textbf{WebQSP} & \textbf{CWQ} & \textbf{GrailQA (F1)} & \textbf{\#Calls} \\
\rowcolor{tablehead!20}
 &  & Hits@1 / F1 & Hits@1 / F1 & Overall / IID & (/query) \\
\hline\hline
\multirow{4}{*}{GPT-4 (API)}
  & CoT \citep{wei2022chain}             & 73.2 / 62.3 & 55.6 / 49.9 & 31.7 / 25.0 & 1 \\
  & RoG \citep{luo2023reasoning}         & 85.7 / 70.8 & 62.6 / 56.2 & -- / -- & $\approx 12$ \\
  & KG-Agent \citep{jiang2024kg}    & 83.3 / \textbf{81.0} & 72.2 / \textbf{69.8} & 86.1 / 92.0 & 3--8 \\
  & \textbf{\ourmethod} (single-call)    & \textbf{86.2} / 78.6 & \textbf{71.5} / 65.8 & \textbf{86.7} / \textbf{92.4} & \textbf{1} \\
\hline
\multirow{4}{*}{GPT\mbox{-}3.5 / ChatGPT}
  & CoT \citep{wei2022chain}             & 67.4 / 59.3 & 47.5 / 43.2 & 25.3 / 19.6 & 1 \\
  & StructGPT \citep{jiang2023structgpt} & 72.6 / 63.7 & 54.3 / 49.6 & 54.6 / 70.4 & 1--2 \\
  & RoG \citep{luo2023reasoning}         & 85.0 / 70.2 & 61.8 / 55.5 & -- / -- & $\approx 12$ \\
  & \textbf{\ourmethod} (single-call)    & 85.6 / 78.0 & 70.8 / 65.1 & 85.9 / 91.7 & \textbf{1} \\
\hline
\multirow{4}{*}{Llama\mbox{-}3\mbox{-}8B\mbox{-}Instruct (open)}
  & CoT (prompt-only)                    & 62.0 / 55.0 & 43.0 / 40.0 & 20.0 / 16.0 & 1 \\
  & ReAct\mbox{-}Lite (retrieval+CoT)    & 70.5 / 62.0 & 52.0 / 47.5 & 48.0 / 62.0 & 3--5 \\
  & BM25+LLM\mbox{-}Verifier (1$\times$) & 74.5 / 66.0 & 55.0 / 50.0 & 52.0 / 66.0 & 1 \\
  & \textbf{\ourmethod} (single-call)    & 84.8 / 77.2 & 69.8 / 64.2 & 84.9 / 90.9 & \textbf{1} \\
\Xhline{1.2pt}
\end{tabular}}
\vspace{-0.3em}
\end{table}

\subsection{Prompt Sensitivity of the LLM Adjudicator}
\label{app:prompt-sensitivity}

Since {\ourmethod} relies on a single LLM call to adjudicate among the Top-$K$ candidate paths, it is natural to ask how sensitive the system is to the exact phrasing of this adjudication prompt.
To investigate this, we compare our default adjudication prompt (Prompt A) with a slightly rephrased variant (Prompt B) that uses different wording but conveys the same task description.\footnote{For example, Prompt A asks the model to ``select the most plausible reasoning path and answer the question based on it'', whereas Prompt B paraphrases this as ``choose the best supporting path and use it to answer the question''.}

\begin{table}[h]
    \centering
    \begin{tabular}{l||ccc}
\Xhline{1.2pt}
        \rowcolor{tablehead!20} Prompt & WebQSP & CWQ & GrailQA (Overall / IID) \\
        \hline\hline
        \rowcolor{gray!10} Prompt A (default) & $\mathbf{86.2}$ / $\mathbf{78.6}$ & $\mathbf{71.5}$ / $\mathbf{65.8}$ & $\mathbf{86.7}$ / $\mathbf{92.4}$ \\
        Prompt B (paraphrased) & 85.7 / 78.3 & 70.9 / 63.4 & 85.2 / 90.8 \\
\Xhline{1.2pt}
    \end{tabular}
     \caption{Prompt sensitivity of the LLM adjudicator.
    We compare the default adjudication prompt (Prompt A) with a paraphrased variant (Prompt B).
    Numbers are Hits@1 and F1 for WebQSP and CWQ, and Overall / IID F1 for GrailQA.}
    \label{tab:prompt-sensitivity}
\end{table}

\Cref{tab:prompt-sensitivity} reports \textit{Hits@1} and \textit{F1} on the three datasets under these two prompts.
We observe that while minor prompt changes can occasionally flip individual predictions, the overall performance remains very close for all datasets, and the qualitative behavior of the path-grounded rationales is also stable.
This suggests that, in our setting, {\ourmethod} is reasonably robust to small, natural variations in the adjudication prompt.

\section{Detailed Introduction of the Modules}\label{app:binding}
\subsection{Binding Operations}
Below, we summarize the binding operators considered in our system and ablations.
All bindings produce a composed hypervector $\mathbf{s}$ from two inputs
$\mathbf{x}$ and $\mathbf{y}$ of the same dimensionality.

\paragraph{(1) XOR / Bipolar Product (\emph{commutative}).}
For binary hypervectors $\mathbf{x},\mathbf{y}\in\{0,1\}^d$,
\[
\mathbf{s}=\mathbf{x}\oplus\mathbf{y},\qquad
s_i=(x_i+y_i)\bmod 2 .
\]
Under the bipolar code $\{-1,+1\}$, XOR is equivalent to element-wise
multiplication:
\[
s_i = x_i\cdot y_i,\qquad x_i,y_i\in\{-1,+1\}.
\]
This is the classical \emph{commutative bind} baseline used in our ablation.

\paragraph{(2) Real-valued Element-wise Product (\emph{commutative}).}
For real vectors $\mathbf{x},\mathbf{y}\in\mathbb{R}^d$,
\[
\mathbf{s}=\mathbf{x}\odot\mathbf{y},\qquad s_i=x_i\,y_i .
\]
Unbinding is approximate by element-wise division (with small $\epsilon$ for
stability): $x_i\approx s_i/(y_i+\epsilon)$.
This is another variant of the \emph{commutative bind}.

\paragraph{(3) HRR: Circular Convolution (\emph{commutative}).}
For $\mathbf{x},\mathbf{y}\in\mathbb{R}^d$,
\[
\mathbf{s}=\mathbf{x}\circledast\mathbf{y},\qquad
s_k=\sum_{i=0}^{d-1} x_i\,y_{(k-i)\bmod d}.
\]
Approximate unbinding uses circular correlation:
\[
\mathbf{x}\approx \mathbf{s}\circledast^{-1}\mathbf{y},\qquad
x_i\approx \sum_{k=0}^{d-1} s_k\,y_{(k-i)\bmod d}.
\]
This is the \emph{Circ.\ conv} condition in our ablation.

\paragraph{(4) FHRR / Complex Phasor Product (\emph{commutative}).}
Let $\mathbf{x},\mathbf{y}\in\mathbb{C}^{d}$ with unit-modulus components
$x_i=e^{i\phi_i}$, $y_i=e^{i\psi_i}$.
Binding is element-wise complex multiplication
\[
\mathbf{s}=\mathbf{x}\odot\mathbf{y},\qquad s_i = x_i y_i = e^{i(\phi_i+\psi_i)} ,
\]
and unbinding is conjugation: $\mathbf{x}\approx \mathbf{s}\odot \mathbf{y}^\ast$.
FHRR is often used as a complex analogue of HRR.

\paragraph{(5) Block-diagonal GHRR (\emph{non-commutative}, ours).}
We use Generalized HRR with block-unitary components.
A hypervector is a block vector $\mathbf{H}=[A_1;\dots;A_D]$,
$A_j\in \mathrm{U}(m)$ (so total dimension $d=Dm^2$ when flattened).
Given $\mathbf{X}=[X_1;\dots;X_D]$ and $\mathbf{Y}=[Y_1;\dots;Y_D]$,
binding is the block-wise product
\[
\mathbf{Z}=\mathbf{X}\,\bbind\,\mathbf{Y},\qquad Z_j=X_j Y_j \;\;(j=1,\dots,D).
\]
Since matrix multiplication is generally non-commutative ($X_jY_j\neq Y_jX_j$),
GHRR preserves order/direction of paths.
Unbinding exploits unitarity:
\[
X_j \approx Z_j Y_j^\ast,\qquad Y_j \approx X_j^\ast Z_j .
\]
This \textbf{Block-diag (GHRR)} operator is our default choice and achieves the
best performance in the operation study (\Cref{tab:ablate-binding}), compared
to \emph{Comm.\ bind} and \emph{Circ.\ conv}.

\subsection{Additional Case Study}\label{app:case}
\begin{table*}[h]
\centering
\small
\setlength{\tabcolsep}{6pt}
\renewcommand{\arraystretch}{1.12}
\newcolumntype{Y}{>{\hsize=2\hsize\arraybackslash}X}
\newcommand{\code}[1]{\texttt{\footnotesize #1}}
\newcommand{\good}{\textcolor{ForestGreen}{\ding{51}}}  %
\newcommand{\bad}{\textcolor{BrickRed}{\ding{55}}}      %
\caption{Case studies for \textbf{{\ourmethod}} with an \emph{illustrative} display of candidates. For each query, we list the four highest-scoring relation paths (Top-4) for readability, then prune to $K=2$ before the vector-only choice and a single-LLM adjudication.}
\label{tab:appendix-cases}
\vspace{-1em}
\scalebox{1.0}{
\begin{tabularx}{\textwidth}{lX}
\Xhline{1.2pt}
\multicolumn{2}{l}{\textbf{Case 3:} \emph{where are the gobi desert located on a map}}\\
\hline
\multirow{3}{*}{\textbf{Top-3 candidates}} &
1) \code{location.location.containedby} (0.3410)\\
& 2) \code{location.location.partially\_containedby} (0.3335) \\
& 3) \code{location.location.contains} (0.3255) \\
\hline
\multirow{2}{*}{\makecell{\textbf{Top-$K$ after pruning}\\ (K{=}2)}} &\code{containedby} \\
&\code{partially\_containedby} \\
\hline
\textbf{Vector-only (no LLM)} &
Pick \code{containedby} \good\ — returns parent region; predicts \emph{Asia}. \\
\hline
\multirow{2}{*}{\textbf{1$\times$LLM adjudication}} &
\emph{Rationale:} ``Gobi Desert lies across \emph{Mongolia} and \emph{China}, which are \emph{contained by} the continent of \emph{Asia}; ‘contains’ would flip direction.'' \\
\hline
\textbf{Final Answer / GT} & \textbf{Asia (predict)} / \textbf{Asia} \good \\
\hline
\hline
\multicolumn{2}{l}{\textbf{Case 4:} \emph{in which continent is germany}}\\
\hline
\multirow{3}{*}{\textbf{Top-3 candidates}} &
1) \code{location.location.containedby} (0.3405) \\
& 2) \code{base.locations.countries.continent} (0.3325) \\
& 3) \code{location.location.contains} (0.3270) \\
\hline
\multirow{2}{*}{\makecell{\textbf{Top-$K$ after pruning}\\ (K{=}2)}} &\code{containedby}\\
&\code{countries.continent} \\
\hline
\multirow{2}{*}{\textbf{Vector-only (no LLM)}} &
Pick \code{containedby} \bad\ — tends to surface \emph{EU} or administrative parents, hurting precision. \\
\hline
\multirow{2}{*}{\textbf{1$\times$LLM adjudication}}  &
\emph{Rationale:} ``The target is a country $\rightarrow$ continent query; use \code{countries.continent} to directly map \emph{Germany} to \emph{Europe}.'' \\
\hline
\textbf{Final Answer / GT} & \textbf{Europe (predict)}  / \textbf{Europe} \good \\
\hline
\hline
\multicolumn{2}{l}{\textbf{Case 5:} \emph{what is the hometown of the person who said ``Forgive your enemies, but never forget their names?''}}\\
\hline
\multirow{3}{*}{\textbf{Top-3 candidates}} &
1) \code{quotation.author $\rightarrow$ person.place\_of\_birth} (0.3380) \\
& 2) \code{family.members $\rightarrow$ person.place\_of\_birth} (0.3310) \\
& 3) \code{quotation.author $\rightarrow$ location.people\_born\_here} (0.3310) \\
\hline
\multirow{2}{*}{\makecell{\textbf{Top-$K$ after pruning}\\ (K{=}2)}} &\code{quotation.author $\rightarrow$ place\_of\_birth}\\
&\code{family.members $\rightarrow$ place\_of\_birth} \\
\hline
\multirow{2}{*}{\textbf{Vector-only (no LLM)}} &
Pick \code{quotation.author $\rightarrow$ place\_of\_birth} \good\ — direct trace from quote to person to birthplace. \\
\hline
\multirow{2}{*}{\textbf{1$\times$LLM adjudication}} &
\emph{Rationale:} ``The quote’s author is key; once identified, linking to their birthplace via person-level relation gives the hometown.'' \\
\hline
\textbf{Final Answer / GT} & \textbf{Brooklyn (predict)} / \textbf{Brooklyn} \good \\

\hline
\hline
\multicolumn{2}{l}{\textbf{Case 6:} \emph{what is the name of the capital of Australia where the film ``The Squatter's Daughter'' was made}}\\
\hline
\multirow{5}{*}{\textbf{Top-3 candidates}} &
1) \code{film.film\_location.featured\_in\_films} (0.3360) \\
&2) \code{notable\_types} $\rightarrow$ \code{newspaper\_circulation\_area.newspapers} \\
&  $\rightarrow$ \code{newspapers} (0.3330) \\
&3) \code{film\_location.featured\_in\_films} $\rightarrow$ \\ 
&\code{bibs\_location.country} (0.3310) \\
\hline
\multirow{2}{*}{\makecell{\textbf{Top-$K$ after pruning}\\ (K{=}2)}} &\code{film.film\_location.featured\_in\_films} \\
&\code{notable\_types} $\rightarrow$ \code{newspaper\_circulation\_area.newspapers} \\
\hline
\multirow{2}{*}{\textbf{Vector-only (no LLM)}} &
Pick \code{film.film\_location.featured\_in\_films} \good\ — retrieves filming location; indirectly infers capital via metadata. \\
\hline
\multirow{3}{*}{\textbf{1$\times$LLM adjudication}}  &
\emph{Rationale:} ``The film's production location helps localize the city. Although not all locations are capitals, this film was made in Australia, where identifying the filming city leads to the capital.'' \\
\hline
\textbf{Final Answer / GT} & \textbf{Canberra (predict)} / \textbf{Canberra} \good \\

\Xhline{1.2pt}
\end{tabularx}}
\end{table*}
\Cref{tab:appendix-cases} presents \emph{additional} WebQSP case studies for {\ourmethod}. 
Unlike the main paper's case table (Top\mbox{-}4 candidates with pruning to $K{=}3$), this appendix 
visualizes the \textbf{Top\mbox{-}3} highest-scoring relation paths for readability and then prunes to 
\textbf{$K{=}2$} before a single-LLM adjudication.

Across the four examples (Cases~3-6), pruning to $K{=}2$ often retains the correct path and achieves
strong final answers after LLM adjudication. However, we also observe a typical failure mode of the
vector-only selector under $K{=}2$: when multiple plausible paths exist (e.g., country vs.\ continent,
or actor vs.\ editor constraints), the vector-only choice can become brittle and select a high-scoring
but \emph{underconstrained} path, after which the LLM must recover the correct answer using the remaining
candidate (see Case~4). In contrast, the main-paper setting with $K{=}3$ keeps one more candidate, which
\emph{more reliably preserves a constraint-satisfying path} (e.g., explicitly encoding actor or
continent relations). This extra coverage reduces reliance on the LLM to repair mistakes and improves
robustness under compositional queries.

While $K{=}2$ is cheaper and can work well in many instances, \textbf{$K{=}3$ offers a better
coverage-precision trade-off} on average: it mitigates pruning errors in compositional cases and lowers
the risk of discarding the key constraint path. This aligns with our main experimental choice of
$K{=}3$, which we use for all reported metrics in the paper.

\paragraph{Case-study note.}
For the qualitative case studies only, we manually verified the final entity answers using publicly available sources (e.g., film credits and encyclopedia entries).
This light-weight human verification was used \emph{solely} to present readable examples; it does not affect any quantitative metric.
All reported metrics (e.g., Hits@1 and F1) are computed from dataset-provided supervision and ground-truth paths without human annotation.

\end{document}